\documentclass[manuscript]{acmart}

\usepackage{xspace}
\usepackage{float}
\usepackage{amsfonts}
\usepackage{amsthm}
\usepackage{mathtools}
\usepackage{subcaption}
\usepackage{multirow}
\usepackage{xcolor}
\usepackage{pifont}
\usepackage{bm}
\usepackage{arydshln}
\usepackage[linesnumbered,ruled,vlined]{algorithm2e}

\AtBeginDocument{%
  }

\setcopyright{acmlicensed}
\copyrightyear{2018}
\acmYear{2018}
\acmDOI{XXXXXXX.XXXXXXX}

\acmJournal{JACM}
\acmVolume{37}
\acmNumber{4}
\acmArticle{111}
\acmMonth{8}

\begin{document}

%%%%%%%%%%%%%%%%%%%%%%%%%% General

\newcommand{\myi}{(\emph{i})\xspace}
\newcommand{\myii}{(\emph{ii})\xspace}
\newcommand{\myiii}{(\emph{iii})\xspace}
\newcommand{\myiv}{(\emph{iv})\xspace}
\newcommand{\myv}{(\emph{v})\xspace}
\newcommand{\myvi}{(\emph{vi})\xspace}
\newcommand{\myvii}{(\emph{vii})\xspace}
\newcommand{\myviii}{(\emph{viii})\xspace}

%% general math
\newcommand{\A}{\mathcal{A}} \newcommand{\B}{\mathcal{B}}
\newcommand{\C}{\mathcal{C}} \newcommand{\D}{\mathcal{D}}
\newcommand{\E}{\mathcal{E}} \newcommand{\F}{\mathcal{F}}
\newcommand{\G}{\mathcal{G}}
\newcommand{\I}{\mathcal{I}} \newcommand{\J}{\mathcal{J}}
\newcommand{\K}{\mathcal{K}}
\newcommand{\LL}{\mathcal{L}}
\newcommand{\M}{\mathcal{M}} \newcommand{\N}{\mathcal{N}}
\newcommand{\PP}{\mathcal{P}}
\newcommand{\Q}{\mathcal{Q}} \newcommand{\R}{\mathcal{R}}
\renewcommand{\S}{\mathcal{S}}
\newcommand{\T}{\mathcal{T}}
\newcommand{\SN}{\mathcal{SN}} 
\newcommand{\SD}{\mathcal{SD}} 
\newcommand{\V}{\mathcal{V}}
\newcommand{\U}{\mathcal{U}}
\newcommand{\W}{\mathcal{W}} \newcommand{\X}{\mathcal{X}}
\newcommand{\Y}{\mathcal{Y}} \newcommand{\Z}{\mathcal{Z}}

\newcommand{\limp}{\mathbin{\rightarrow}}
\newcommand{\incl}{\subseteq}
\newcommand{\ind}{\hspace*{.18in}}

%% LTL
\newcommand{\Wnext}{\raisebox{-0.27ex}{\LARGE$\bullet$}}
\newcommand{\Next}{\raisebox{-0.27ex}{\LARGE$\circ$}}
\newcommand{\Until}{\mathop{\U}}
\newcommand{\Release}{\mathop{\R}}
\newcommand{\Wuntil}{\mathop{\W}}
\newcommand{\trueVal}{\mathit{true}} %% semantic 'true'
\newcommand{\falseVal}{\mathit{false}} %% semantic 'false'
\newcommand{\ttrue}{\mathit{tt}} %% syntax 'tt'
\newcommand{\ffalse}{\mathit{ff}} %% syntax 'ff'
\newcommand{\final}{\mathit{Final}}
\newcommand{\Last}{\mathit{last}}
\newcommand{\Ended}{\mathit{end}}
\newcommand{\length}{\mathit{length}}
\newcommand{\last}{\mathit{n}}
\newcommand{\nnf}{\mathit{NNF}}
\newcommand{\CL}{\mathit{CL}}
\newcommand{\MU}[2]{\mu #1.#2}
\newcommand{\NU}[2]{\nu #1.#2}
\newcommand{\BOX}[1]{ [#1]}
\newcommand{\DIAM}[1]{\langle #1 \rangle}
\newcommand{\transl}{f}
\newcommand{\Int}[2][\I]{#2^{#1}}
\newcommand{\INT}[2][\I]{(#2)^{#1}}
\newcommand{\Inta}[2][\rho]{#2_{#1}^\I}
\newcommand{\INTA}[2][\rho]{(#2)_{#1}^\I}
%% Logics
\newcommand{\LTL}{{\sf LTL}\xspace}
\newcommand{\ltl}{{\sf LTL}\xspace}
\newcommand{\LTLf}{{\sf LTL}$_f$\xspace}
\newcommand{\ltlf}{{\sf LTL}$_f$\xspace}
\newcommand{\ldlf}{{\sf LDLf}\xspace}
\newcommand{\SoneS}{{\sf S1S}\xspace}
\newcommand{\fol}{{\sf FOL}\xspace}
\newcommand{\NFA}{{\sf NFA}\xspace}
\newcommand{\DFA}{{\sf DFA}\xspace}
\newcommand{\TDFA}{{\sf TDFA}\xspace}
\newcommand{\dfa}{{\sf DFA}\xspace}
\newcommand{\tdfa}{{\sf TDFA}\xspace}

\newcommand{\NNF}{{\sc NNF}\xspace}
\newcommand{\XNF}{{\sc XNF}\xspace}
\newcommand{\BNF}{{\sc BNF}\xspace}

%% Circled numbers 1-10
%% \usepackage{pifont}
\newcommand{\circled}[1]{\raisebox{-0.06ex}{\ding{\numexpr 191 + #1\relax}}\xspace}

%% Abbreviations
\newcommand{\buchi}{B\"uchi\xspace}
\newcommand{\Nat}{{\rm I\kern-.23em N}}
\newcommand{\Prop}{\PP}
\newcommand{\Var}{\V}

\newcommand{\PreC}{\mathit{PreC}}
\newcommand{\Win}{\mathit{Win}}

\newcommand{\Futrue}{\mathop{\F}}
\newcommand{\Global}{\mathop{\G}}
\newcommand{\Endt}{\mathit{end}}
\newcommand{\atomize}[1]{\texttt{"}\ensuremath{#1}\texttt{"}}

\newcommand{\ot}{o}

\def\tool{\textsf{tool}\xspace}

\newcommand{\xnf}[1]{\textsf{xnf}(#1)}
\newcommand{\tnf}[1]{\textsf{tnf}(#1)}

\newcommand{\tail}{\mathit{tail}}
\newcommand{\SET}[1]{\{ #1 \}}
\newcommand{\tup}[1]{\langle #1 \rangle}

\newcommand{\qbf}{{\sc QBF}\xspace}

\newcommand{\subAUT}{\mathop{\preceq_\A}}

\newcommand{\angleBra}[1]{\left\langle #1 \right\rangle}
\newcommand{\boxBra}[1]{\left[ #1 \right]}
\newcommand{\roundBra}[1]{\left( #1 \right)}
\newcommand{\curlyBra}[1]{\left\{ #1 \right\}}

\newcommand{\tran}[1]{\xrightarrow[]{#1}}
\newcommand{\fp}[1]{{\sf fp}(#1)}
\newcommand{\cl}[1]{{\sf cl}(#1)}
\newcommand{\tcl}[1]{{\sf tcl}(#1)}
\newcommand{\reach}[1]{\textsf{Reach}(#1)}

\def\approach{\textsf{MoGuS}\xspace}

\def\SAT{\textsf{SAT}\xspace}
\def\SDD{\textsf{SDD}\xspace}
\def\BDD{\textsf{BDD}\xspace}
\def\DFS{\textsf{DFS}\xspace}
\def\SCC{\textsf{SCC}\xspace}
\def\DAG{\textsf{DAG}\xspace}
\def\API{\textsf{API}\xspace}

\def\aaltaf{\textsf{aaltaf}\xspace}
\def\tool{{\sf Tople}\xspace}
\def\olfs{{\sf OLFS}\xspace}
\def\cynthia{{\sf Cynthia}\xspace}
\def\nike{{\sf Nike}\xspace}
\def\lisa{{\sf Lisa}\xspace}
\def\lydia{{\sf LydiaSyft}\xspace}
\def\spot{{\sf Spot}\xspace}
\def\syft{{\sf Syft}\xspace}
\def\mona{{\sf MONA}\xspace}
\def\ltlsynt{{\sf ltlsynt}\xspace}
\def\strix{{\sf Strix}\xspace}
\def\acacia{{\sf Acacia+}\xspace}

\def\todo{\textcolor{red}{TODO: }\xspace}
\newcommand{\os}[1]{\textcolor{cyan}{OS: #1}}
\newcommand{\JL}[1]{\textcolor{red}{JL: #1}}
\newcommand{\shengping}[1]{\textcolor{violet}{shengping: #1}}

%%
%% The "title" command has an optional parameter,
%% allowing the author to define a "short title" to be used in page headers.
\title{On-the-fly Synthesis for \LTL over Finite Traces: An Efficient Approach that Counts}

%%
%% The "author" command and its associated commands are used to define
%% the authors and their affiliations.
%% Of note is the shared affiliation of the first two authors, and the
%% "authornote" and "authornotemark" commands
%% used to denote shared contribution to the research.
\author{Shengping Xiao}
\email{spxiao@stu.ecnu.edu.cn}
\affiliation{%
  \institution{East China Normal University}
  \city{Shanghai}
  \country{China}
}
\author{Yongkang Li}
\email{51265902012@stu.ecnu.edu.cn}
\affiliation{%
  \institution{East China Normal University}
  \city{Shanghai}
  \country{China}
}
\author{Shufang Zhu}
\email{shufang.zhu@cs.ox.ac.uk}
\affiliation{%
  \institution{University of Oxford}
  \city{London}
  \country{UK}
}
\author{Jun Sun}
\email{junsun@smu.edu.sg}
\affiliation{%
  \institution{Singapore Management University}
  \country{Singapore}
}
\author{Jianwen Li}
\email{jwli@sei.ecnu.edu.cn}
\affiliation{%
  \institution{East China Normal University}
  \city{Shanghai}
  \country{China}
}
\author{Geguang Pu}
\email{ggpu@sei.ecnu.edu.cn}
\affiliation{%
  \institution{East China Normal University}
  \city{Shanghai}
  \country{China}
}
\author{Moshe Y. Vardi}
\email{vardi@cs.rice.edu}
\affiliation{%
  \institution{Rice University}
  \city{Houston}
  \country{USA}
}
\renewcommand{\shortauthors}{Xiao et al.}

%%
%% The abstract is a short summary of the work to be presented in the
%% article.
\begin{abstract}
We present an on-the-fly synthesis framework for Linear Temporal Logic over finite traces~(\ltlf) based on top-down deterministic automata construction. Existing approaches rely on constructing a complete Deterministic Finite Automaton (\DFA) corresponding to the \ltlf specification, a process with doubly exponential complexity relative to the formula size in the worst case. In this case, the synthesis procedure cannot be conducted until the entire \dfa is constructed. This inefficiency is the main bottleneck of existing approaches. To address this challenge, we first present a method for converting \LTLf into Transition-based \DFA (\TDFA) by directly leveraging \LTLf semantics, incorporating intermediate results as direct components of the final automaton to enable parallelized synthesis and automata construction. We then explore the relationship between \LTLf synthesis and \TDFA games and subsequently develop an algorithm for performing \LTLf synthesis using on-the-fly \TDFA game solving. This algorithm traverses the state space in a global forward manner combined with a local backward method, along with the detection of strongly connected components. Moreover, we introduce two optimization techniques --- model-guided synthesis and state entailment --- to enhance the practical efficiency of our approach. Experimental results demonstrate that our on-the-fly approach achieves the best performance on the tested benchmarks and effectively complements existing tools and approaches.
\end{abstract}

%%
%% The code below is generated by the tool at http://dl.acm.org/ccs.cfm.
%% Please copy and paste the code instead of the example below.
%%
\begin{CCSXML}
<ccs2012>
   <concept>
       <concept_id>10011007.10010940.10010992.10010998</concept_id>
       <concept_desc>Software and its engineering~Formal methods</concept_desc>
       <concept_significance>500</concept_significance>
       </concept>
   <concept>
       <concept_id>10003752.10003790.10003794</concept_id>
       <concept_desc>Theory of computation~Automated reasoning</concept_desc>
       <concept_significance>500</concept_significance>
       </concept>
   <concept>
       <concept_id>10003752.10003790.10003793</concept_id>
       <concept_desc>Theory of computation~Modal and temporal logics</concept_desc>
       <concept_significance>500</concept_significance>
       </concept>
 </ccs2012>
\end{CCSXML}

\ccsdesc[500]{Software and its engineering~Formal methods}
\ccsdesc[500]{Theory of computation~Automated reasoning}
\ccsdesc[500]{Theory of computation~Modal and temporal logics}

%%
%% Keywords. The author(s) should pick words that accurately describe
%% the work being presented. Separate the keywords with commas.
\keywords{Synthesis, Realizability, \LTL over Finite Traces, Reactive System}

\received{20 February 2007}
\received[revised]{12 March 2009}
\received[accepted]{5 June 2009}

%%
%% This command processes the author and affiliation and title
%% information and builds the first part of the formatted document.
\maketitle

\SetKwRepeat{Do}{do}{while}
\SetKwProg{myproc}{function}{}{}

\SetKw{Break}{break}
\SetKw{Continue}{continue}
\SetKw{Init}{initialize}
\SetKw{Stack}{stack}

\SetKwFunction{isSccRoot}{isSccRoot}
\SetKwFunction{getScc}{getScc}
\SetKwFunction{isRealizable}{isRealizable}
\SetKwFunction{currentSystemWinning}{currentSystemWinning}
\SetKwFunction{currentEnvironmentWinning}{currentEnvironmentWinning}
\SetKwFunction{isRealizable}{isRealizable}
\SetKwFunction{pushBack}{pushBack}
\SetKwFunction{pushBack}{pushBack}
\SetKwFunction{popBack}{popBack}
\SetKwFunction{back}{back}
\SetKwFunction{insert}{insert}
\SetKwFunction{remove}{remove}
\SetKwFunction{checkCurrentStatus}{checkCurrentStatus}
\SetKwFunction{forwardSearch}{forwardSearch}
\SetKwFunction{backwardSearch}{backwardSearch}
\SetKwFunction{getEdge}{getEdge}
\SetKwFunction{directPredecessors}{directPredecessors}
\SetKwFunction{ltlfSat}{ltlfSat}
\SetKwFunction{getModel}{getModel}
\SetKwFunction{edgeConstraint}{edgeConstraint}
\SetKwFunction{noSwinPotential}{noSwinPotential}

\section{Introduction}

Formal \emph{synthesis} aims to generate a system from a specification in a manner that is correct by construction, i.e., ensuring the resulting system satisfies the specification. Compared with formal verification, in which case both a specification and an implementation must be provided, synthesis automatically derives the latter from the former, with a correctness guarantee. Synthesis thus indicates a fundamental paradigm shift, transitioning the construction of reactive systems from the imperative to the declarative level. In situations where a comprehensive specification is available prior to system implementation, synthesis presents a natural and promising endeavor. Furthermore, reactive synthesis emerges as a specialized synthesis process focused on the automatic creation of reactive systems based on given specifications, such that the interactive behaviors between the system and the adversarial environment are guaranteed to satisfy the specification.

The decidability of the synthesis problem, i.e., determining the existence of a system that implements a given specification, is called \emph{realizability}. Synthesis and realizability were initially introduced by Church~\cite{Chu62}, who used Monadic Second-Order Logic of One Successor (\SoneS) as a specification language. The \SoneS synthesis and realizability problems were solved in the 1960s~\cite{Rab69,BL69}, with a non-elementary complexity. In recent years, temporal logics, especially Linear Temporal Logic (\ltl)~\cite{Pnu77}, have proven to be a valuable and more modern tool for specifying system behaviors and expressing temporal properties. \LTL synthesis and realizability have become an active research area in formal methods, and fruitful works about \ltl synthesis have been established on both the theoretical and practical aspects, e.g., \cite{FJR09,BJPPS12,MC18,MSL18}, to name a few. 

\LTL formulas are interpreted over \emph{infinite} traces, which are suitable for describing the infinite behaviors of non-terminating systems. Meanwhile, reasoning about temporal constraints or properties over a \emph{finite}-time horizon is useful and important in various areas, such as business processes~\cite{VPS09,PBV10}, robotics~\cite{LAFKV15}, and user preferences~\cite{BFM11}. To this end, a variant of \ltl with an adaptation to finite-trace semantics, called \ltlf, has emerged as a popular logic, especially in AI-related domains since its introduction~\cite{GV13}. Extensive works have studied fundamental theoretical and practical problems of \LTLf, e.g., satisfiability checking~\cite{LRPZV19} and translation to automata~\cite{SXLGP20,DF21}. In this work, we focus on the problem of synthesizing systems from specifications in \ltlf, which was first introduced in~\cite{GV15}. 

We study the problem of reactive synthesis in the context of synchronous reactive systems, which maintain a constant interaction with the environment. From the perspective of the system, variables involved in the specification can be partitioned into \emph{inputs} and \emph{outputs}: outputs are under the control of the system, while inputs are assigned by the environment adversarially. Therefore, at each time point, we need to find some output that can adequately handle all possible inputs. More specifically, given every possible input sequence, solving the synthesis problem is to find an output sequence that induces an execution satisfying the specification.

% \textcolor{red}{SZ: the following two paragraphs about DFA-based LTLf synthesis are very unclear. For instance, the discussions of the papers from Suguman and Marco are not enough. There should be a clear difference on the papers from Shufang (direct DFA construction), Suguman (decomposition only on the conjunction level) and Marco (decomposition on the syntax level).}

Automata-based approaches play a key role in solving \ltlf synthesis, which is the same as in other reactive synthesis problems such as \LTL synthesis~\cite{BCJ18}. Every \ltlf formula can be converted into an equivalent Deterministic Finite Automaton (\dfa) recognizing the same language~\cite{GV13}, and \ltlf synthesis and realizability can be solved by reducing to suitable two-player games specified by the corresponding \dfa~\cite{GV15}. The process of solving \LTLf synthesis can be broken down into two parts. Firstly, we construct an equivalent \dfa from the \ltlf specification. Secondly, a reachability two-player game on the \dfa, which is considered as the game arena, is solved by a \emph{backward} fixed-point computation. To solve this reachability game, we iteratively refine an under-approximation of the set of winning states (for the system) of the \DFA. Initially, the set of winning states is simply the set of accepting states. After that, it is expanded into a larger winning set after each iteration until no more states can be added. Upon reaching the fixed point, if the initial state is in the winning set, we conclude that the \LTLf formula is realizable and return a \emph{winning strategy}; conversely, the \LTLf formula is unrealizable. Hereafter, we refer to this approach as `backward search'. As established in~\cite{GV13,GV15}, an \ltlf formula can be translated into the corresponding \dfa in 2EXPTIME, a reachability game on a \dfa can be solved in linear time in the number of states of the automaton, and the asymptotic complexity of \ltlf synthesis is 2EXPTIME-complete.

Following the aforementioned theoretical framework, several efforts have been made in developing practical \ltlf synthesis approaches, which make various attempts to improve the \DFA construction process. S. Zhu et al.~\cite{ZTLPV17} transform \ltlf formulas into equivalent formulas in first-order logic (\fol) on finite words~\cite{Tho97} and employ the off-the-shelf \fol-to-\dfa translator MONA~\cite{MonaManual2001} for \dfa construction. Subsequent works \cite{TV19,BLTV20} decompose the \ltlf formula on the conjunction level and construct the corresponding \dfa for the conjuncts, respectively. G. De Giacomo et al. \cite{DF21} decompose the \ltlf formula on the thorough syntax level, translate \LTLf formulas into formulas in Linear Dynamic Logic over finite traces~\cite{GV13} (\ldlf), then build the \dfa based on the \ldlf semantics. All these studies collectively show that the \DFA construction is the main bottleneck of \LTLf synthesis in both theory and practice.

The primary challenge with the above synthesis techniques is the \DFA construction from the \LTLf specification, which has a doubly exponential complexity. Consequently, the question arises: is it feasible to solve the \LTLf synthesis problem without generating the entire \DFA? An intuitive idea is to perform the synthesis procedure \emph{on-the-fly}. In this work, we introduce an alternative framework for \ltlf synthesis, employing direct and \emph{top-down} automata construction and solving the corresponding automata games in a \emph{global forward} and \emph{local backward} manner. In top-down automata construction from \ltlf formulas, the states and transitions computed in the process are parts of the final automaton. This enables us to possibly decide the realizability before reaching the worst-case double-exponential complexity.

In detail, we present a technique to create the transition-based \dfa (\tdfa) from \LTLf specifications. This technique produces the automata based on the semantics of \LTLf, which utilizes \emph{formula progression}~\cite{Eme90,BK98} to separate what happens in the \emph{present} (edge) and what would happen in the \emph{future} (successor). It iteratively invokes a `primitive', which builds a deterministic transition from a current state each time.
To determine the realizability of the \ltlf specification and synthesize a system, we solve a \tdfa game on the fly while generating the \tdfa. Solving a \tdfa game means to classify the states of the corresponding \tdfa into system-winning states and environment-winning states, from which the system/environment can win the game. Once the initial state is determined as a system-winning or environment-winning state, the realizability result of the original problem can be concluded. If realizable, it would also return a winning strategy. Typically, the strategy may be partial, as it does not require traversing the entire state space in the on-the-fly synthesis.

From a global perspective, we perform a forward depth-first search. We search the state space starting from the initial state (the input \ltlf formula here) and proceed forward to computing the remaining states of the automata as necessary. For each of the states, we check whether it is system-winning or environment-winning. Once a state is determined, a backtracking procedure is invoked to check whether the predecessors are system/environment-winning accordingly. However, since there may be loops in the automaton transition system, we may not have collected all the necessary information upon the first visit to a state. This results in the occurrences of undetermined states during the global forward search. Consequently, a local backward search is also required. During the depth-first traversal, we employ Tarjan's algorithm~\cite{tarjan1972SCC} to detect strongly connected components ({\SCC}s). We show that we have already obtained all the necessary information required to resolve a current \SCC when the forward search of states within the \SCC is completed, though there may still be states left undetermined at that time. At this point, a local backward search is conducted within the scope of the current \SCC to finalize the states left undetermined by the forward search.

Furthermore, we propose two optimizations to make the above on-the-fly synthesis framework more practical. Firstly, we guide the synthesis procedure with satisfiable traces of the \ltlf specification to achieve a more targeted search. In \ltlf synthesis, the potential choices for both the system and the environment can be exponentially many. Hence, a random search in the above on-the-fly synthesis is unlikely to be efficient. Considering the winning condition of the games is reaching accepting states/transitions, we adopt a greedy strategy: each time we need to move forward, we choose the edge that leads to one of the nearest accepting transitions from the current state. Secondly, we investigate the semantic entailment relations among system-winning states, as well as among environment-winning states. If a newly computed state is semantically entailed by a known system-winning state or semantically entails a known environment-winning state, we can accordingly classify it as system-winning or environment-winning.

We implement our approach in a tool called \tool and conduct comprehensive experiments by comparing it with state-of-the-art \LTLf synthesis tools. Experimental results show that \tool and the on-the-fly approach achieve the best overall performance, though they cannot completely subsume other tools and approaches. Additionally, we also demonstrate the effectiveness of the two optimization techniques through an ablation study.

\paragraph{Origin of the Results} On-the-fly \ltlf synthesis was first presented in~\cite{XLZSPV21}, and model-guided \ltlf synthesis was introduced in~\cite{XLHXLPSV24}. Beyond these, the new contributions of this work include:
\begin{itemize}

\item Providing the full proofs for lemmas and theorems previously introduced in conference papers.

\item Formulating the {\sf (T)DFA} game and discussing its relationship with \ltlf synthesis in more detail.

\item Completing the on-the-fly approach by proposing a solution to better handle loops in the automaton transition system --- detecting {\SCC}s.

\item Introducing the optimization technique of state entailment.

\end{itemize} 

\section{Preliminaries}

A sequence of some elements is denoted by $\zeta = \zeta[0], \zeta[1], \cdots, \zeta[k], \cdots$. The elements can be propositional interpretations, states of an automaton, or elements from a set, and we denote by $\E$ the set of elements in general. A sequence can either be finite (i.e., $\zeta \in\E^*$) or infinite (i.e., $\zeta\in\E^\omega$). $\zeta[i]$ ($0 \leq i < |\zeta|$) is the $i$-th element of $\zeta$. $|\zeta|$ represents the length of $\zeta$, and we have $|\zeta|=\infty$ if $\zeta$ is infinite. We use $\zeta^i$ to represent $\zeta[0],\zeta[1],\cdots,\zeta[i-1]$ ($i\geq 1$), which is the prefix of $\zeta$ up to position $i$ (excluding $i$), and $\zeta_i$ to represent $\zeta[i],\zeta[i+1],\cdots,\zeta[n]$, which is the suffix of $\zeta$ from position $i$ (including $i$). The empty sequence is denoted by $\epsilon$. We have $\zeta^i=\epsilon$ and $\zeta_j=\epsilon$ when $i\leq 0$ and $j\geq|\zeta|$ hold, respectively.
% Two finite traces, $\rho_1$ and $\rho_2$, can be concatenated to one trace $\rho$, denoted by $\rho=\rho_1\cdot\rho_2$.

\subsection{Linear Temporal Logic over Finite Traces}
Linear Temporal Logic over finite traces, or \ltlf~\cite{GV13}, extends propositional logic with finite-horizon temporal connectives. Generally speaking, \ltlf is a variant of Linear Temporal Logic~(\LTL)~\cite{Pnu77} that is interpreted over finite traces.
\subsubsection{\ltlf Syntax} Given a set of atomic propositions $\PP$, the syntax of \ltlf is identical to \LTL, and defined as:
\begin{equation}
  \varphi ::= \ttrue \mid p \mid \neg \varphi \mid \varphi \wedge \varphi \mid \Next \varphi \mid \varphi \Until \varphi\text{,}
\end{equation}
where $\ttrue$ represents the \emph{true} formula, $p \in \PP$ is an atomic proposition, $\neg$ represents \emph{negation}, $\wedge$ represents \emph{conjunction}, $\Next$ represents the \emph{strong Next} operator and $\Until$ represents the \emph{Until} operator. We also have the corresponding dual operators $\ffalse$~(\emph{false}) for $\ttrue$, $\vee$~(disjunction) for $\wedge$, $\Wnext$~(weak Next) for $\Next$ and $\Release$~(Release) for $\Until$. Moreover, we use the notation $\G\varphi$ (Global) and $\F\varphi$ (Future) to represent $\ffalse \Release\varphi$ and $\ttrue \Until\varphi$, respectively. Notably, $\Next$ is the standard \emph{Next} (also known as \emph{strong Next}) operator, while $\Wnext$ is \emph{weak Next}; $\Next$ requires the existence of a successor instant, while $\Wnext$ does not. Thus $\Wnext\phi$ is always true in the last instant of a finite trace since no successor exists there.
A literal is an atom $p \in \PP$ or its negation ($\neg p$). An \ltlf formula is in the negation normal form if the negation operator appears only in front of atomic propositions. Every \LTLf formula can be converted into an equivalent \LTLf formula in negation normal form in linear time. Hereafter, we assume that all \LTLf formulas are in negation normal form.
% We define the closure $cl(\varphi)$ of an \ltlf formula $\varphi$ as the set comprising all subformulas of which the top-level operator is not a Boolean connective. Formally, we have 
% \begin{alignat*}{2}
%     &cl(\varphi)=\curlyBra{\varphi}\text{ if }\varphi=\ttrue/\ffalse/p &\quad\quad &cl(\Next \varphi)=\curlyBra{\Next \varphi}\cup cl(\varphi) \\[-1.5ex]
%     &cl(\neg p)=\curlyBra{p} &&cl(\Wnext\varphi)=\curlyBra{\Wnext\varphi}\cup cl(\varphi)\\[-1.5ex]
%     &cl(\varphi_1\wedge\varphi_2) = cl(\varphi_1)\cup cl(\varphi_2) &&cl(\varphi_1\Until\varphi_2)=\curlyBra{\varphi_1\Until\varphi_2}\cup cl(\varphi_1)\cup cl(\varphi_2)\\[-1.5ex]
%     &cl(\varphi_1\vee\varphi_2) = cl(\varphi_1)\cup cl(\varphi_2) &&cl(\varphi_1\Release\varphi_2)=\curlyBra{\varphi_1\Release\varphi_2}\cup cl(\varphi_1)\cup cl(\varphi_2)\text{,}
% \end{alignat*}
% where $p\in\PP$ is a proposition.

\subsubsection{\ltlf Semantics} A \emph{finite trace} $\rho = \rho[0],\rho[1],\cdots,\rho[n]\in(2^\PP)^*$ is a finite sequence of propositional interpretations. Intuitively, $\rho[i]$ is interpreted as the set of propositions that are $\trueVal$ at instant $i$. \ltlf formulas are interpreted over finite traces. For a finite trace $\rho$ and an \ltlf formula $\varphi$, we define the satisfaction relation $\rho\models\varphi$ (i.e., $\rho$ is a model of $\varphi$) as follows:
\begin{itemize}
	\item 
	$\rho \models \ttrue$ and $\rho \not\models \ffalse$;
	\item 
	$\rho \models p$ iff $p \in \rho[0]$, and $\rho \models \neg p$ iff $p \notin \rho[0]$, where $p\in\PP$ is an atomic proposition;
	\item 
	$\rho \models \varphi_1 \wedge \varphi_2$ iff $\rho \models \varphi_1$ and $\rho \models \varphi_2$;
    \item 
	$\rho \models \varphi_1 \vee \varphi_2$ iff $\rho \models \varphi_1$ or $\rho \models \varphi_2$;
	\item 
	$\rho \models \Next \varphi$ iff $|\rho|>1$ and $\rho_1 \models \varphi$;
    \item 
	$\rho \models \Wnext \varphi$ iff $|\rho|=1$ or $\rho_1 \models \varphi$;
    \item 
	$\rho \models \varphi_1 \Until \varphi_2$ iff there exists $i$ with $0\leq i < |\rho|$ such that \circled{1} $\rho_i\models \varphi_2$ holds, and \circled{2} for every $j$ with $0 \leq j < i$ it holds that $\rho_j \models \varphi_1$;
    \item 
	$\rho \models \varphi_1 \Release \varphi_2$ iff there exists $i$ with $0\leq i < |\rho|$ such that \circled{1} $i=|\rho|-1$ or $\rho_i\models\varphi_1$ holds, and \circled{2} for every $j$ with $0 \leq j \leq i$ it holds that $\rho_j \models \varphi_2$.
\end{itemize}
The set of finite traces that satisfy \LTLf formula $\varphi$ is the language of $\varphi$, denoted as $\LL(\varphi)=\{\rho\in(2^\PP)^+\mid \rho\models\varphi\}$.
Two \LTLf formulas $\varphi_1$ and $\varphi_2$ are semantically equivalent, denoted as $\varphi_1\equiv\varphi_2$, iff for every finite trace $\rho$, $\rho\models\varphi_1$ iff $\rho\models\varphi_2$. And $\varphi_1$ semantically entails $\varphi_2$, denoted as $\varphi_1\Rightarrow\varphi_2$, iff for every finite trace $\rho$, $\rho\models\varphi_2$ if $\rho\models\varphi_1$. And $\varphi_1\not\Rightarrow\varphi_2$ denotes that semantic entailment does not hold.
From the \ltlf semantics, we have the following lemma.
\begin{lemma}\label{lem:UR-equiv}
    For arbitrary \ltlf formulas $\varphi_1$ and $\varphi_2$, it holds that
    \begin{itemize}
        \item $\varphi_1\Until\varphi_2 \equiv \varphi_2\vee(\varphi_1\wedge\Next(\varphi_1\Until\varphi_2))$;
        \item $\varphi_1\Release\varphi_2 \equiv \varphi_2\wedge(\varphi_1\vee\Wnext(\varphi_1\Release\varphi_2))$.
    \end{itemize}
\end{lemma}

\subsubsection{Propositional Semantics for \ltlf} In addition to the language-based semantics of \ltlf, we assign propositional semantics to \ltlf formulas by treating temporal subformulas as propositional variables, which plays a useful role in linear temporal reasoning as in \cite{LRPZV19,EKS20,GFLVXZ22}. A formula $\varphi$ is classified as \emph{temporal} if it is neither a conjunction nor a disjunction, i.e., either it is a literal or the root of its syntax tree is labelled by a temporal operator ($\Next$, $\Wnext$, $\Until$, or $\Release$). We represent the set of temporal subformulas of $\varphi$ as $\tcl{\varphi}$ (signifying `\underline{t}emporal \underline{cl}osure'). Formally, given a set of \LTLf formulas $\I$ and an \LTLf formula $\varphi$, the propositional satisfaction relation $\I\models_p\varphi$ is defined as follows:
\begin{itemize}
    \item $\I\models_p\ttrue$ and $\I\not\models_p\ffalse$;
    \item $\I\models_p l$ iff $l\in\I$, where $l$ is a literal;
	\item $\I\models_p\varphi_1 \wedge \varphi_2$ iff $\I\models_p\varphi_1$ and $\I\models_p\varphi_2$;
	\item $\I\models_p\varphi_1 \vee \varphi_2$ iff $\I\models_p\varphi_1$ or $\I\models_p\varphi_2$;
    \item $\I\models_p \mathop{op}\varphi$ iff $\mathop{op}\varphi\in\I$, where $\mathop{op}\in\curlyBra{\Next,\Wnext}$;
    \item $\I\models_p \varphi_1\mathop{op}\varphi_2$ iff $\varphi_1\mathop{op}\varphi_2\in\I$, where $\mathop{op}\in\curlyBra{\Until,\Release}$.
\end{itemize}
Two formulas $\varphi$, $\psi$ are propositionally equivalent, denoted $\varphi\sim\psi$, if for arbitrary formula sets $\I$, $\I\models_p \varphi$ iff $\I \models_p \psi$ holds. The propositional equivalence class of a formula $\varphi$ is denoted by $[\varphi]_\sim$ and defined as $[\varphi]_\sim=\curlyBra{\psi\mid\varphi\sim\psi}$. The propositional quotient set of a set of formulas $\Phi$ is denoted $\Phi/\sim$ and defined as $\Phi/\sim=\curlyBra{[\varphi]_\sim\mid\varphi\in\Phi}$. The following are two properties of propositional semantics.
\begin{lemma}\label{lem:prop-1}
    Given a function $f$ on formulas such that $f(\ttrue)=\ttrue$, $f(\ffalse)=\ffalse$, and $f(\chi_1\wedge\chi_2)=f(\chi_1) \wedge f(\chi_1)$, $f(\chi_1 \vee \chi_2) = f(\chi_1) \vee f(\chi_2)$ for all formulas $\chi_1$ and $\chi_1$, for every \ltlf formulas $\varphi$ and $\psi$, if $\varphi\sim\psi$, then $f(\varphi) \sim f(\psi)$.
\end{lemma}
\begin{proof}
A proof for this lemma within the context of \LTL is provided in \cite{EKS20}, which can be readily extended to \LTLf. We here offer a brief proof sketch. First, for every formula $\varphi$ and every assignment $\I$, we can obtain the following result by a structural induction on $\varphi$.
\begin{equation}\label{eq:lem-prop-1-proof}
    \I\models_p f(\varphi) \Leftrightarrow \{\chi\in\tcl{\varphi}\mid\I\models_p f(\chi)\}\models_p\varphi
\end{equation}
Second, we demonstrate that $\I\models_p f(\varphi) \Rightarrow\I\models_p f(\psi)$ holds for every assignment $\I$, which suffices to support Lemma~\ref{lem:prop-1} by symmetry.
\begin{align*}
    \I\models_p f(\varphi)&\Leftrightarrow \{\chi\in\tcl{\varphi}\mid\I\models_p f(\chi)\}\models_p\varphi &(\text{Equation~(\ref{eq:lem-prop-1-proof})})\\
    &\Leftrightarrow \{\chi\in\tcl{\varphi}\mid\I\models_p f(\chi)\}\models_p\psi &(\varphi\sim\psi)\\
    &\Rightarrow \{\chi\in\tcl{\varphi}\cap\tcl{\psi}\mid\I\models_p f(\chi)\}\models_p\psi &\\
    &\Rightarrow \{\chi\in\tcl{\psi}\mid\I\models_p f(\chi)\}\models_p\psi &\\
    &\Leftrightarrow \I\models_p f(\psi) &(\text{Equation~(\ref{eq:lem-prop-1-proof})})
\end{align*}
\end{proof}
\begin{lemma}\label{lem:prop-2}
    Given two \ltlf formulas $\varphi_1$ and $\varphi_2$ such that $\varphi_1\sim\psi$, $\rho\models\varphi$ iff $\rho\models\psi$ holds for every finite trace $\rho$.
\end{lemma}
\begin{proof}
We here offer a brief proof sketch. First, for every formula $\varphi$ and every finite trace $\rho$, we can obtain the following result by a structural induction on $\varphi$.
\begin{equation}\label{eq:lem-prop-2-proof}
    \rho\models\varphi \Leftrightarrow \{\chi\in\tcl{\varphi}\mid\rho\models \chi\}\models_p\varphi
\end{equation}
Second, we demonstrate that $\rho\models\varphi \Rightarrow\rho\models\psi$ holds for every finite trace $\rho$, which suffices to support Lemma~\ref{lem:prop-2} by symmetry.
\begin{align*}
    \rho\models\varphi \Leftrightarrow& \{\chi\in\tcl{\varphi}\mid\rho\models \chi\}\models_p\varphi &(\text{Equation~(\ref{eq:lem-prop-2-proof})})\\
    \Leftrightarrow& \{\chi\in\tcl{\varphi}\mid\rho\models \chi\}\models_p\psi &(\varphi\sim\psi)\\
    \Rightarrow& \{\chi\in\tcl{\varphi}\cap\tcl{\psi}\mid\rho\models \chi\}\models_p\psi &\\
    \Rightarrow& \{\chi\in\tcl{\psi}\mid\rho\models \chi\}\models_p\psi &\\
    \Leftrightarrow& \rho\models\psi&(\text{Equation~(\ref{eq:lem-prop-2-proof})})
\end{align*}
\end{proof}

Lemma~\ref{lem:prop-2} indicates that for any finite trace $\rho$, if $\rho\models\varphi$ and $\psi\in[\varphi]$, then $\rho\models\psi$ holds. We then generalize the language-based satisfaction relation to propositional equivalence classes. Given an \ltlf formulas $\varphi$ and a finite trace $\rho$, 
 we have $\rho\models[\varphi]_\sim$ iff $\rho\models\varphi$.

\subsection{\ltlf Synthesis}

In the context of reactive \emph{synthesis}, a comprehensive \emph{specification} should encompass essential information, including the properties that the targeted system is required to satisfy, the partition of variables controlled separately by the \emph{system} and the \emph{environment}, and the type of the targeted system, i.e., either Mealy or Moore. Formally, an \LTLf specification is a tuple $(\varphi,\X, \Y)_t$, where $\varphi$ is an \LTLf formula over propositions in $\X\cup\Y$; $\X$ is the set of input variables controlled by the environment, $\Y$ is the set of output variables controlled by the system, and it holds that $\X\cap\Y = \emptyset$; and $t\in\curlyBra{Mealy, Moore}$ is the type of the target system.
In a reactive system, interactions happen in turns where both the system and the environment assign values to their respective controlled variables. The order of assignment within each turn determines the system types: if the environment assigns values first, it creates a Mealy machine; conversely, if the system assigns values first, it results in a Moore machine.

\begin{definition}[\ltlf Realizability and Synthesis]\label{def:syn}

An \ltlf specification $\roundBra{\varphi,\X,\Y}_t$ is realizable, iff:
\begin{itemize}
\item $\bm{t=Mealy.}$ \quad there exists a winning strategy $g: (2^\X)^+ \to 2^\Y$ such that for an arbitrary infinite sequence $\lambda = X_0, X_1, \cdots \in (2^\X)^\omega$ of propositional interpretations over $\X$, there is $k \geq 0$ such that $\rho\models\varphi$ holds, where $\rho=(X_0\cup g(X_0)),(X_1\cup g(X_0,X_1)),\cdots,(X_k\cup g(X_0,\cdots,X_{k}))$;

\item $\bm{t=Moore.}$ \quad there exists a winning strategy $g: (2^\X)^* \to 2^\Y$ such that for an arbitrary infinite sequence $\lambda = X_0, X_1, \cdots \in (2^\X)^\omega$ of propositional interpretations over $\X$, there is $k \geq 0$ such that $\rho\models\varphi$ holds, where $\rho=(X_0\cup g(\epsilon)),(X_1\cup g(X_0)),\cdots,(X_k\cup g(X_0,\cdots,X_{k-1}))$.
\end{itemize}
The \ltlf realizability problem is to determine whether an \ltlf specification is realizable, and the \ltlf synthesis problem for a realizable specification is to compute a winning strategy. An \LTLf specification is unrealizable when it is not realizable.
\end{definition}

Synthesis for constructing Mealy systems is also referred to as \emph{environment-first} synthesis, while synthesis for constructing Moore systems is termed \emph{system-first} synthesis. The methodologies presented in this paper are applicable to both environment-first and system-first synthesis. For brevity, we only focus on the system-first synthesis in the main text.

\subsection{Transition-Based \DFA}

The Transition-based Deterministic Finite Automaton~(\TDFA) is a variant of the Deterministic Finite Automaton~\cite{SXLGP20}.
\begin{definition}[\TDFA]\label{def:tdfa}
A transition-based \DFA (\TDFA) is a tuple $\A=(2^\PP,S,init,\delta,T)$ where
\begin{itemize}
    \item $2^\PP$ is the alphabet;
    \item $S$ is the set of states;
    \item $init\in S$ is the initial state;
    \item $\delta:S\times 2^\PP\to S$ is the transition function;
    \item $T\subseteq (S\times(2^\PP))$ is the set of accepting transitions.
\end{itemize}
\end{definition}

The run $r$ of a finite trace $\rho=\rho[0], \rho[1],\cdots, \rho[n] \in (2^\PP)^+$ on a \TDFA $\A$ is a finite state sequence $r = s_0,s_1,\cdots,s_n$ such that $s_0=init$ is the initial state, $\delta(s_i,\rho[i])=s_{i+1}$ holds for $0\leq i<n$. Note that runs of \TDFA do not need to include the destination state of the last transition, which is implicitly indicated by $s_{n+1}=\delta (s_n, \rho[n])$, since the starting state ($s_n$) together with the labels of the transition ($\rho[n]$) are sufficient to determine the destination.
% $r$ is called \emph{acyclic} iff $(s_i=s_j)\Leftrightarrow (i=j)$ for $0\leq i,j< n$. Also, we say that $\rho$ \emph{runs across} $s_i$ iff $s_i$ is in the corresponding run $r$.
The trace $\rho$ is accepted by $\A$ iff the corresponding run $r$ ends with an accepting transition, i.e., $(s_n,\rho[n])\in T$. The set of finite traces accepted by a \TDFA $\A$ is the language of $\A$, denoted as $\LL(\A)$.
According to~\cite{SXLGP20}, \TDFA has the same expressiveness as \DFA.%, and for an \ltlf formula $\varphi$, there is a \TDFA $\A_\varphi$ such that $\LL(\varphi)=\LL(\A_\varphi)$. %As a result, the \ltlf satisfiability-checking problem can be solved on the corresponding \TDFA. That is, an \ltlf formula $\varphi$ is satisfiable iff there is a finite trace accepted by its corresponding \TDFA $\A_{\varphi}$~\cite{SXLGP20}.

\subsection{Strongly Connected Components of a Directed Graph}

Let $G=(V,E)$ be a directed graph. A component $C\subseteq V$ of $G$ is strongly connected if a directed path exists between any two vertices. $C\subseteq V$ is a Strongly Connected Component, abbreviated as \SCC, if $C$ is strongly connected and no proper superset of $C$ is strongly connected. The set of {\SCC}s of $G$ forms a Directed Acyclic Graph (\DAG) that has an edge from $C_0$ to $C_1$ when there exists a vertex in $C_1$ reachable from a vertex in $C_0$. Tarjan’s algorithm~\cite{tarjan1972SCC} finds all the {\SCC}s of $G$ in a single depth-first traversal of $G$, visiting each vertex just once, with a complexity of $O(|V|+|E|)$. These components are detected in the reverse order of the topological sort of the \DAG formed by the components. Specifically, if $(C_0,C_1)$ is an edge of this \DAG, then $C_1$ is found before $C_0$.
Omitting the details of Tarjan’s algorithm, we assume that it supports the following {\API}s along with a depth-first traversal:
\begin{itemize}
\item $\isSccRoot(v)$ checks whether $v\in V$ is the root vertex of the current \SCC. This function should be invoked when all successors of $v$ have been processed with a depth-first strategy and it backtracks to $v$ for the last time. If it returns $\trueVal$, it indicates that we have just finished traversing a \SCC.

\item $\getScc()$ retrieves the \SCC that has been detected just now after the $\isSccRoot(v)$ call returns $\trueVal$.
\end{itemize}

\section{Automata Construction}\label{sec:tdfa}

{\sf (T)DFA} construction plays a foundamental role in solving the \ltlf synthesis problem. We now present a method for converting \ltlf into \tdfa using \emph{formula progression}. This approach directly leverages \ltlf semantics to expand the transition system. Additionally, the intermediate results generated during the construction process integrate into the final automaton, so that it can support on-the-fly \LTLf synthesis.

Formula progression has been investigated in prior studies~\cite{Eme90,BK98}. Here we extend this concept specifically to \tdfa construction from \ltlf. Given an \LTLf formula $\varphi$ over $\PP$ and a finite trace $\rho\in(2^\PP)^+$, to make $\rho\models \varphi$, we can start from $\varphi$ to progress through $\rho$. We consider \LTLf formula $\varphi$ into a requirement about the \emph{present} $\rho[0]$, which can be verified straightforwardly, and a requirement about the \emph{future} that needs to hold for the yet unavailable suffix $\rho_1$. Essentially, formula progression looks at $\varphi$ and $\rho[0]$, and progresses a new formula $\fp{\varphi,\rho[i]}$ such that $\rho\models\varphi$ iff $\rho_{1}\models \fp{\varphi,\rho[0]}$. This procedure is analogous to an automaton reading a trace, where it reaches accepting transitions by moving forward step by step.

\begin{definition}[Formula Progression for \ltlf]\label{def:fprog}
Given an \ltlf formula $\varphi$ and a non-empty finite trace $\rho$, the progression formula $\fp{\varphi, \rho}$ is recursively defined as follows:
\begin{itemize}
    \item $\fp{\ttrue, \rho}=\ttrue$ and $\fp{\ffalse, \rho}=\ffalse$;
    
    \item $\fp{p, \rho} = \ttrue$ if $p\in\rho[0]$; $\fp{p, \rho} = \ffalse$ if $p\notin\rho[0]$;

    \item $\fp{\neg p, \rho} = \ffalse$ if $p\in\rho[0]$; $\fp{\neg p, \rho} = \ttrue$ if $p\notin\rho[0]$;
    
    \item $\fp{\varphi_1\wedge\varphi_2, \rho} = \fp{\varphi_1, \rho}\wedge\fp{\varphi_2, \rho}$;
    
    \item $\fp{\varphi_1\vee\varphi_2, \rho} = \fp{\varphi_1, \rho}\vee\fp{\varphi_2, \rho}$;
    
    \item $\fp{\Next\varphi, \rho} = \varphi$ if $|\rho|=1$; $\fp{\Next\varphi, \rho} = \fp{\varphi, \rho_1}$ if $|\rho|>1$;
    
    \item $\fp{\Wnext\varphi, \rho} = \varphi$ if $|\rho|=1$; $\fp{\Wnext\varphi, \rho} = \fp{\varphi, \rho_1}$ if $|\rho|>1$;
    
    \item $\fp{\varphi_1 \Until\varphi_2, \rho}=\fp{\varphi_2, \rho} \vee  (\fp{\varphi_1, \rho} \wedge \fp{\Next(\varphi_1 \Until\varphi_2), \rho})$;
    
    \item $\fp{\varphi_1 \Release\varphi_2, \rho}= \fp{\varphi_2, \rho}  \wedge (\fp{\varphi_1, \rho} \vee \fp{\Wnext(\varphi_1 \Release\varphi_2), \rho})$.
\end{itemize}
\end{definition}

Notice that a letter $\sigma\in2^\PP$ is interpreted as a trace of length $1$. So the second parameter of the function $\fp{\varphi,\rho}$ can be either a trace or a letter. Additionally, when applied to a specific trace $\rho$, the function $\fp{\varphi,\rho}$ satisfies the conditions stated in Lemma~\ref{lem:prop-1}. Hence, we can generalize \ltlf formula progression from single formulas to propositional equivalence classes:
\begin{equation}
  \fp{[\varphi]_\sim,\rho}=[\fp{\varphi,\rho}]_\sim=[\fp{\psi,\rho}]_\sim\text{,}
\end{equation}
where $\psi\in[\varphi]_\sim$.

The following lemma indicates that Definition~\ref{def:fprog} is aligned with our expectations of decomposing an \ltlf formula into requirements about the present and the future.
\begin{lemma}\label{lem:fprog}
    Given an \LTLf formula $\varphi$ and a finite trace $\rho\in(2^\PP)^+$ with $|\rho|>1$, it holds that
    \begin{enumerate}
        \item $\rho\models\varphi$ iff $\rho_1\models\fp{\varphi,\rho[0]}$, and\label{item:lem-fprog-1}
        \item $\rho\models[\varphi]_\sim$ iff $\rho_1\models\fp{[\varphi]_\sim,\rho[0]}$.\label{item:lem-fprog-2}
    \end{enumerate}
\end{lemma}
\begin{proof}
We focus on proving (\ref{item:lem-fprog-1}) here, while (\ref{item:lem-fprog-2}) follows (\ref{item:lem-fprog-1}) directly from the definition of \ltlf language-based satisfaction relation on propositional equivalence classes.

We establish the proof by structural induction over the syntax composition of $\varphi$. 

\textbf{Basis.}
\begin{itemize}
\item $\varphi=\ttrue/\ffalse$.
By Definition~\ref{def:fprog}, we have $\fp{\varphi,\rho[0]}=\fp{\ttrue,\rho[0]}=\ttrue/\ffalse$. The lemma holds obviously.
    
\item $\varphi=p$.

$\rho\models p
  \Leftrightarrow%{\ltlf\text{ semantics}}
p\in\rho[0]
  \xLeftrightarrow{\text{Definition~\ref{def:fprog}}}
\fp{p,\rho[0]}=\ttrue
  \xLeftrightarrow{\text{Definition~\ref{def:fprog}}}
\rho_1\models\fp{p,\rho[0]}
$.

\item $\varphi=\neg p$.

$\rho\models\neg p
  \Leftrightarrow%{\ltlf\text{ semantics}}
p\notin\rho[0]
  \xLeftrightarrow{\text{Definition~\ref{def:fprog}}}
\fp{\neg p,\rho[0]}=\ttrue
  \xLeftrightarrow{\text{Definition~\ref{def:fprog}}}
\rho_1\models\fp{\neg p,\rho[0]}
$.
\end{itemize}

\textbf{Induction.}
\begin{itemize}
    
\item $\varphi=\varphi_1\wedge\varphi_2$, the induction hypotheses are $\rho\models\varphi_1$ iff $\rho_1\models\fp{\varphi_1,\rho[0]}$ and $\rho\models\varphi_2$ iff $\rho_1\models\fp{\varphi_2,\rho[0]}$.
\begin{align*}
\rho\models\varphi_1\wedge\varphi_2 \Leftrightarrow& \rho\models\varphi_1\text{ and }\rho\models\varphi_2 &(\text{\ltlf semantics})\\
\Leftrightarrow& \rho_1\models\fp{\varphi_1,\rho[0]}\text{ and }\rho_1\models\fp{\varphi_2,\rho[0]} &(\text{induction hypotheses})\\
\Leftrightarrow& \rho_1\models\fp{\varphi_1,\rho[0]}\wedge\fp{\varphi_2,\rho[0]} &(\text{\ltlf semantics})\\
\Leftrightarrow& \rho_1\models\fp{\varphi_1\wedge\varphi_2,\rho[0]} &(\text{Definition~\ref{def:fprog}})
\end{align*}

\item $\varphi=\varphi_1\vee\varphi_2$. The proof for this case follows analogously to $\varphi=\varphi_1\wedge\varphi_2$.

\item $\varphi=\Next\psi$. Notice that $|\rho|>1$.

$\rho\models\Next\psi
  \xLeftrightarrow{\text{\ltlf semantics}}
\rho_1\models\psi
  \xLeftrightarrow{\text{Definition~\ref{def:fprog}}}
\rho_1\models\fp{\Next\psi,\rho[0]}
$.

\item $\varphi=\Wnext\psi$. Notice that $|\rho|>1$.

$\rho\models\Wnext\psi
  \xLeftrightarrow{\text{\ltlf semantics}}
\rho_1\models\psi
  \xLeftrightarrow{\text{Definition~\ref{def:fprog}}}
\rho_1\models\fp{\Next\psi,\rho[0]}
$.

\item $\varphi=\varphi_1\Until\varphi_2$, the induction hypotheses are $\rho\models\varphi_1$ iff $\rho_1\models\fp{\varphi_1,\rho[0]}$ and $\rho\models\varphi_2$ iff $\rho_1\models\fp{\varphi_2,\rho[0]}$.
\begin{align*}
\rho\models\varphi_1\Until\varphi_2 \Leftrightarrow& \rho\models\varphi_2\vee(\varphi_1\wedge\Next(\varphi_1\Until\varphi_2)) &(\text{Lemma}~\ref{lem:UR-equiv})\\
\Leftrightarrow&
\begin{cases}
    \rho\models\varphi_2\text{, or}\\
    \rho\models\varphi_1\text{ and }\rho\models\Next(\varphi_1\Until\varphi_2)
\end{cases}
&(\text{\ltlf semantics})\\
\Leftrightarrow&
\begin{cases}
    \rho_1\models\fp{\varphi_2,\rho[0]}\text{, or}\\
    \rho_1\models\fp{\varphi_1,\rho[0]}\text{ and }\rho\models\Next(\varphi_1\Until\varphi_2)
\end{cases}
&(\text{induction hypotheses})\\
\Leftrightarrow&
\begin{cases}
    \rho_1\models\fp{\varphi_2,\rho[0]}\text{, or }\\
    \rho_1\models\fp{\varphi_1,\rho[0]}\text{ and }\rho_1\models\varphi_1\Until\varphi_2
\end{cases}
&(\text{\ltlf semantics})\\
\Leftrightarrow&\rho_1\models\fp{\varphi_2, \rho[0]} \vee  (\fp{\varphi_1, \rho[0]} \wedge (\varphi_1\Until\varphi_2) &(\text{\ltlf semantics})\\
\Leftrightarrow&\rho_1\models\fp{\varphi_1\Until\varphi_2,\rho[0]} &(\text{Definition~\ref{def:fprog}})
\end{align*}

\item $\varphi=\varphi_1\Release\varphi_2$, the induction hypotheses are $\rho\models\varphi_1$ iff $\rho_1\models\fp{\varphi_1,\rho[0]}$ and $\rho\models\varphi_2$ iff $\rho_1\models\fp{\varphi_2,\rho[0]}$.
\begin{align*}
\rho\models\varphi_1\Release\varphi_2 \Leftrightarrow& \rho\models\varphi_2\wedge(\varphi_1\vee\Wnext(\varphi_1\Release\varphi_2)) &(\text{Lemma}~\ref{lem:UR-equiv})\\
\Leftrightarrow&
\begin{cases}
    \rho\models\varphi_2\text{, and}\\
    \rho\models\varphi_1\text{ or }\rho\models\Wnext(\varphi_1\Release\varphi_2)
\end{cases}
&(\text{\ltlf semantics})\\
\Leftrightarrow&
\begin{cases}
    \rho_1\models\fp{\varphi_2,\rho[0]}\text{, and }\\
    \rho_1\models\fp{\varphi_1,\rho[0]}\text{ or }\rho\models\Next(\varphi_1\Release\varphi_2)
\end{cases}
&(\text{induction hypotheses})\\
\Leftrightarrow&
\begin{cases}
    \rho_1\models\fp{\varphi_2,\rho[0]}\text{, and}\\
    \rho_1\models\fp{\varphi_1,\rho[0]}\text{ or }\rho_1\models\varphi_1\Release\varphi_2
\end{cases}
&(\text{\ltlf semantics})\\
\Leftrightarrow&\rho_1\models\fp{\varphi_2, \rho[0]} \wedge  (\fp{\varphi_1, \rho[0]} \vee (\varphi_1\Release\varphi_2) &(\text{\ltlf semantics})\\
\Leftrightarrow&\rho_1\models\fp{\varphi_1\Release\varphi_2,\rho[0]} &(\text{Definition~\ref{def:fprog}})
\end{align*}
\end{itemize}
\end{proof}

With formula progression, we can translate \ltlf formulas to \tdfa. Firstly, we define the states of the automaton as propositional equivalence classes, with the initial state being the equivalence class corresponding to the source formula. Then, formula progression enables us to calculate the successor state (i.e., $\fp{s,\sigma}$) that is reached by a given state $s$ through an edge $\sigma$. In building the transition system of \tdfa $\A_\varphi$, we start from the initial state $[\varphi]_\sim$, applying formula progression iteratively to determine all states and transitions, which is similar to a standard graph traversal. Finally, the accepting condition is defined on transitions where the edge directly satisfies the current state.

\begin{definition}[\ltlf to \TDFA]\label{def:ltlf2tdfa}
Given an \ltlf formula $\varphi$, the \TDFA ${\A_{\varphi}}$ is a tuple $\roundBra{2^\PP, S, \delta, init, T}$ such that
\begin{itemize}
	\item $2^{\PP}$ is the alphabet, where $\PP$ is the set of atoms of $\varphi$;

	\item $S=\reach{\varphi}/\sim$ is the set of states, where $\reach{\varphi}=\curlyBra{\varphi}\cup\{ \fp{\varphi, \rho} \mid \rho \in (2^\PP)^{+} \}$; 
	
	\item $init = [\varphi]_\sim$ is the initial state;

	\item $\delta:  S\times 2^\PP \to S$ is the transition function such that $\delta(s,\sigma)=\fp{s,\sigma}$ for $s\in S$ and $\sigma\in2^\PP$;
	
	\item $T=\{(s,\sigma)\in S\times2^\PP\mid \sigma\models s\}$ is the set of accepting transitions.
\end{itemize}
\end{definition}

The correctness of our \tdfa construction is stated in the theorem below.

\begin{theorem}\label{thm:tdfaCorrectness}
Given an \LTLf formula $\varphi$ and the \TDFA $\A_\varphi$ constructed by Definition~\ref{def:ltlf2tdfa}, it holds that $\rho\models\varphi$ iff $\rho\in\LL(\A_\varphi)$ for any finite trace $\rho$.
\end{theorem}

\begin{proof}
Recall that we have $\rho\models[\varphi]_\sim$ iff $\rho\models\varphi$. To prove this theorem, it is only necessary to prove that $\rho\models[\varphi]_\sim$ iff $\rho\in\LL(\A_\varphi)$ for any finite trace $\rho$.
Assume $|\rho|=n+1$ with $n\geq0$.

$(\Rightarrow)$ We perform this part of the proof in two cases over the value of $n$.
\begin{itemize}
    \item If $n=0$ (i.e., $|\rho|=1$), $\rho\models[\varphi]_\sim$ is equivalent to $\rho[0]\models init$. By Definition~\ref{def:ltlf2tdfa}, $(init,\rho[0])$ is an accepting transition. So it holds that $\rho$ is accepted by $\A_\varphi$, i.e., $\rho\in\LL(\A_\varphi)$.
    
    \item If $n>0$, there exists a run $r=s_0, \cdots,$ $s_n, s_{n+1}$ over $\A_\varphi$ on $\rho$ such that $s_0=init$ and $\delta(s_i,\rho[i])=s_{i+1}$ holds for $0\leq i\leq n$. We first prove that $\rho_i\models s_i$ for $0\leq i\leq n$ by induction over the value of $i$.
    
    \textbf{Basis.} $\rho\models[\varphi]_\sim$ i.e., $\rho_0\models s_0$ holds basically. 
    
    \textbf{Induction.} The induction hypothesis is $\rho_{i}\models s_{i}$ holds for $0\leq i<n$. By Lemma~\ref{lem:fprog} and the induction hypothesis, we have $\rho_{i+1}\models\fp{s_i,\rho[i]}$. By Definition~\ref{def:ltlf2tdfa}, it holds that $s_{i+1}=\fp{s_i,\rho[i]}$, which implies $\rho_{i+1}\models s_{i+1}$.
    
    Therefore, we have $\rho_n\models s_n$. Note that $\rho_n=\rho[n].$ By Definition~\ref{def:ltlf2tdfa}, $s_n\times\rho[n]\to s_{n+1}$ is an accepting transition. So it holds that $\rho$ is accepted by $\A_\varphi$, i.e., $\rho\in\LL(\A_\varphi)$.
\end{itemize}

$(\Leftarrow)$ $\rho\in\LL(\A_\varphi)$ implies that there exists a run $r=s_0, s_1, \cdots,$ $s_n, s_{n+1}$ over $\A_\varphi$ on $\rho$ such that $s_0=init$ and $\delta(s_i,\rho[i])=s_{i+1}$ holds for $0\leq i\leq n$ and we have $\rho[n]\models s_n$.

We first prove that $\rho\in\LL(\A_\varphi)$ implies $\rho_{n-i}\models s_{n-i}$ holds for $0\leq i\leq n$ by induction over the value of $i$.

\textbf{Basis.}
By Definition~\ref{def:ltlf2tdfa}, $\rho\in\LL(\A_\varphi)$ implies that the corresponding run ends with accepting transition $s_n\times\rho[n]\to s_{n+1}$. Then we have $\rho[n]\models s_n$. So it holds that $\rho_{n-0}\models s_{n-0}$ when $i=0$.
    
\textbf{Induction.}
The induction hypothesis is that $\rho_{n-i}\models s_{n-i}$ holds for $0<i\leq n$. By Definition~\ref{def:ltlf2tdfa}, we have $\delta(s_{n-(i+1)},\rho[n-(i+1)])=s_{n-i}=\fp{s_{n-(i+1)},\rho[n-(i+1)]}$. So the induction hypothesis is equivalent to $\rho_{n-i}\models\fp{s_{n-(i+1)},\rho[n-(i+1)]}$. By Lemma~\ref{lem:fprog}, we have $\rho_{n-(i+1)}\models s_{n-(i+1)}$.

Now we have proved that $\rho\in\LL(\A_\varphi)$ implies $\rho_{n-i}\models s_{n-i}$ holds for $0\leq i\leq n$. Then we can have $\rho\models[\varphi]_\sim$ (i.e., $\rho_{n-n}\models s_{n-n}$) holds when $i=n$.
\end{proof}

The size of \tdfa constructed by Definition~\ref{def:ltlf2tdfa} depends on
the properties of propositional equivalence. States computed via formula progression do not create new temporal subformulas.

\begin{lemma}\label{lem:fp-tcl}
    Given an \LTLf formula $\varphi$ and a finite trace $\rho$, it holds that $\tcl{\fp{\varphi,\rho}}\subseteq\tcl{\varphi}$.
\end{lemma}
\begin{proof}
This lemma holds intuitively, because $\fp{\varphi,\rho}$ does not introduce new elements in the syntax tree, except for Boolean combinations of existing temporal subformulas. The formal proof can be performed via a straightforward nested structural induction on the syntax of $\varphi$ and the length of $\rho$.
\end{proof}

The worst-case complexity of constructing deterministic automata using this method is also doubly exponential.

\begin{theorem}\label{thm:tdfaComplexity}
    Given an \LTLf formula $\varphi$, $\A_\varphi$ is a \tdfa constructed by Definition~\ref{def:ltlf2tdfa}. The state set of $\A_\varphi$ has a cardinality in $O(2^{2^{n}})$, i.e., $|S|=|\reach{\varphi}/\sim|=O(2^{2^{n}})$, where $n=|\tcl{\varphi}|$.
\end{theorem}
\begin{proof}
By Lemma~\ref{lem:fp-tcl}, every formula of in $\reach{\varphi}$ is a Boolean combination of formulas in $\tcl{\varphi}$. So each equivalence class in $\reach{\varphi}/\sim$ can be interpreted as a Boolean function over $n$ variables, which can be bounded by $2^{2^{n}}$, considering that there exist at most $2^{2^{n}}$ Boolean functions over $n$ variables.
\end{proof}

\ltlf formulas in the same propositional equivalence class exhibit semantic equivalence, and further, they are equivalent with regards to satisfiability and realizability. A single formula can uniquely determine a propositional equivalence class. When implementing Definition~\ref{def:ltlf2tdfa}, it is neither necessary nor possible to explicitly compute the propositional equivalence class. Instead, we use a formula to represent an equivalence class. Upon deriving a new formula by formula progression (Definition~\ref{def:fprog}), we can check whether it is propositionally equivalent to an already computed formula. If it is, then they are classified within the same equivalence class and thus represent the same state of the automaton. Conversely, if not equivalent, the new formula is a new automaton state. To simplify the notations, we omit the propositional equivalence symbol $[]_\sim$, so that an \LTLf formula can represent either the formula itself or the corresponding propositional equivalence class, depending on the context.

\section{\tdfa Games}

We now introduce two-player games on \tdfa, which is considered as an invariant of traditional infinite reachability/safety graph-based games~\cite{BCJ18}. Moreover, most of the concepts and results presented in this section have their origins in, or correspond to, those found in infinite reachability/safety graph-based games. From a theoretical perspective, this section studies \TDFA games in terms of formalism, solution, and their correlation with \LTLf synthesis.

\subsection{What a \tdfa Game Is}
The \tdfa \emph{game} is played on a \tdfa with the initial state removed. Every \tdfa game involves two \emph{players}, the environment and the system, controlling the variables in sets $\X$ and $\Y$ respectively. Each \emph{round} of the game consists of a state and the environment and system setting the values of the propositions they control. The winner is determined by the winning condition, i.e., the accepting transitions $T$ here. The system wins if the play reaches an accepting transition. The parameter $t\in\{Mealy, Moore\}$ specifies the order in which players assign values in each round. In alignment with \LTLf synthesis, we focus on the situation of $t=Moore$ in the main text, where the system assigns values first in each round.
\begin{definition}[\tdfa Game]
A \tdfa game is a tuple $\G=(2^{\X\cup\Y},S,\delta,T)_t$, where
\begin{itemize}
    \item $2^{\X\cup\Y}$ is the alphabet with $\X\cap\Y=\emptyset$;
    \item $S$ is the set of states;
    \item $\delta:S\times 2^\PP\to S$ is the transition function;
    \item $T\subseteq S\times2^\PP$ is the set of accepting transitions;
    \item $t\in\curlyBra{Mealy,Moore}$ is the type of the system.
\end{itemize}
\end{definition}

A \emph{play} of the game starts in some state and progresses through various states as the system and environment alternately assign values.

\begin{definition}[Play]
A play is a sequence of rounds $p=(s_0,X_0\cup Y_0), \cdots, (s_n,X_n\cup Y_n), \cdots \in (S\times2^{\X\cup\Y})^* \cup(S\times2^{\X\cup\Y})^\omega$ such that $s_{i+1} = \delta(s_i,X_i\cup Y_i)$ holds for $0\leq i<|p|$, where the sequence may be finite or infinite. If the system wins, the play terminates; the play continues infinitely and the environment wins.
\end{definition}

Plays in \tdfa games can be categorized into three sets described below. This shows a key difference of \TDFA games from traditional infinite games~\cite{BCJ18}, as it involves both finite and infinite plays.
\begin{itemize}
    \item Ongoing plays. Play $p$ is an ongoing play if $p=(s_0,X_0\cup Y_0), \cdots, (s_n,X_n\cup Y_n)\in(S\times2^{\X\cup\Y})^*$ is a finite play such that none of rounds in $p$ is an accepting transition, i,e., $(s_i,X_i\cup Y_i)\notin T$ holds for $0\leq i\leq n$;
    
    \item System-winning plays. Play $p$ is a system winning play if $p=(s_0,X_0\cup Y_0), \cdots, (s_n,X_n\cup Y_n)\in(S\times2^{\X\cup\Y})^*$ is a finite play such that $p$ ends with an accepting transition and none of the previous rounds is an accepting transition, i,e., $(s_n,X_n\cup Y_n)\in T$ and $(s_i,X_i\cup Y_i)\notin T$ holds for $0\leq i< n$;

    \item Environment-winning plays. Play $p$ is an environment winning play if $p=(s_0,X_0\cup Y_0), (s_1,X_1 \cup Y_1), \cdots \in (S\times2^{\X\cup\Y})^\omega$ is an infinite play such that none of rounds in $p$ is an accepting transition, i,e., $(s_i,X_i\cup Y_i)\notin T$ holds for $i\in\mathbb{N}$.
\end{itemize}

In the context of games, a \emph{strategy} provides the decisions for a player based on the history of the play. Given a \tdfa game $(2^{\X\cup\Y},S,\delta,T)_{Moore}$, a system strategy is a function $\pi:S^+\to2^\Y$, and an environment strategy is a function $\tau: S^+\times2^\Y\to2^\X$.
A play $p=(s_0,X_0\cup Y_0), \cdots,(s_n,X_n\cup Y_n),\cdots$ is \emph{consistent} with a system strategy $\pi$ if $Y_i=\pi(s_0,\cdots,s_i)$ holds for $0\leq i< |p|$; similarly, $p$ is consistent with an environment strategy $\tau$ if $X_i=\tau(s_0,\cdots,s_i,Y_i)$ holds for $0\leq i< |p|$. Notably, starting from a state, the strategies $\pi$ and $\tau$ of the two players together uniquely identify one specific play.

Given a \tdfa game $\G$, a system strategy $\pi$ is a system-winning strategy from a state $s$ if there exists no environment-winning play consistent with $\pi$ starting in $s$, and an environment strategy $\tau$ is an environment-winning strategy from a state $s$ if there exists no system-winning play consistent with $\tau$ starting in $s$.

\begin{definition}[Winning Strategy]
Given a \tdfa game $\G=(2^{\X\cup\Y},S,\delta,T)_{Moore}$,
\begin{itemize}
\item a system strategy $\pi:S^+\to2^\Y$ is a system-winning strategy from a state $s_0\in S$ if for any sequence $\lambda=X_0,X_1,\cdots\in(2^\X)^\omega$, there exists $k\geq0$ such that play $p$ is a system-winning play, where
$p=(s_0,X_0\cup\pi(s_0)),\cdots,(s_k,X_k\cup\pi(s_0,\cdots,s_k))$;

\item  an environment strategy $\tau:S^+\times2^\Y\to2^\X$ is an environment-winning strategy from a state $s_0\in S$ if for any sequence $\lambda=Y_0,Y_1,\cdots\in(2^\Y)^\omega$, play $p$ is an environment-winning play, where $p=(s_0,\tau(s_0,Y_0)\cup Y_0),\cdots,(s_k,\tau(s_0,\cdots,s_k,Y_k)\cup Y_k),\cdots$.
\end{itemize}
\end{definition}

\begin{definition}[Winning State]
A state $s$ is a system-winning state if there exists a system-winning strategy from $s$. A state $s$ is an environment-winning state if there exists an environment-winning strategy from $s$. %The set of system winning states in $\G$ is called winning region for the system.
\end{definition}
It is easy to see that no state can be winning for both players.
\begin{lemma}\label{lem:win-insertion}
    Given a \tdfa game $\G=(2^{\X\cup\Y},S,\delta,T)_{Moore}$, there exists no state that can be both system-winning and environment-winning.
\end{lemma}
\begin{proof}

Assume that there exists a state $s_0\in S$ and strategies $\pi$ and $\tau$ that are winning from $s$ for the system and environment, respectively. Then there exists a unique infinite sequence $\zeta$ determined by $\pi$ and $\tau$.
\begin{equation}
    \zeta=(s_0,\pi(s_0)\cup\tau(s_0,\pi(s_0))),\cdots,(s_k,\pi(s_0,\cdots,s_k)\cup\tau(s_0,\cdots,s_k,\pi(s_0,\cdots,s_k))),\cdots
\end{equation}
When $\pi$ is a system-winning strategy from $s_0$, then there exists $i\in\mathbb{N}$ such that $\zeta[i]\in T$ is an accepting transition. On the other hand, when $\tau$ is an environment-winning strategy from $s_0$, then for every $i\in\mathbb{N}$, $\zeta[i]\notin T$ is not an accepting transition. The assumption leads to a contradiction since above both two situations cannot hold simultaneously.
\end{proof}

Note that a winning strategy in the game is state-dependent. If a system strategy is winning for all system-winning states, we call it a \emph{uniform} winning strategy. In addition, the above definition of a strategy is general in the sense that the decisions are based on the entire history of the play. However, it often suffices to work with simpler strategies. A system strategy is \emph{positional} iff $\pi(uv)=\pi(v)$ holds for all $u\in S^*$, $v\in S$. Uniform strategies and positional strategies for the environment can be defined similarly.

\subsection{How to Solve a \tdfa Game}
Given a \tdfa game $\G=(2^{\X\cup\Y},S,\delta,T)_{Moore}$, we define the construction of the set of winning states for $n\in\mathbb{N}$ as follows.
\begin{align}
    CPre^S(\E)&=
        \{s\in S\mid\exists Y\in2^\Y.\forall X\in2^\X.(s,X\cup Y)\in T\text{ or }\delta(s,X\cup Y)\in\E\}\label{eq:cpres}\\
    CPre^E(\E)&=
        \{s\in S\mid\forall Y\in2^\Y.\exists X\in2^\X.(s,X\cup Y)\notin T\text{ and }\delta(s,X\cup Y)\in\E\}\label{eq:cpree}
\end{align}
\begin{align}
    SWin^0&=\emptyset&\quad SWin^{i+1}&=SWin^i\cup CPre^S(SWin^i)\label{eq:swin}
    \\
     EWin^0&=S&\quad EWin^{i+1}&=EWin^i\cap CPre^E(EWin^i)\label{eq:ewin}
\end{align}
The abbreviation `$CPre^S$' and `$CPre^E$' mean `\underline{C}ontrollable \underline{Pre}decessors of the \underline{S}ystem/\underline{E}nvironment'. The sets $SWin_i$ and $EWin_i$ respectively represent an under-approximation of system-winning states and an over-approximation of environment-winning states. Intuitively, $SWin^i$ is the set of states where the system can enforce the play to reach accepting transitions in $i$ steps, and $EWin^i$ is the set of states where the environment can enforce the play to avoid accepting transitions in $i$ steps.

We rewrite the recurrence relation of $SWin^i$ as follows:
\begin{equation}\label{eq:fs}
    SWin^{i+1}=f_s(SWin^i)=SWin^i\cup CPre^S(SWin^i)\text{.}
\end{equation}
$f_s:2^S\to2^S$ is a monotonic function over the complete lattice $(2^S,\subseteq)$. According to Equations~(\ref{eq:swin}) and (\ref{eq:fs}), we start from the basic situation $SWin^0$ and iteratively compute $SWin^i$ until a fixed point $SWin^{n+1}=SWin^n=f_s(SWin^n)$ is reached. The existence of such a fixed point is guaranteed by the Knaster-Tarski fixed-point theorem~\cite{tarski1955lattice}. We denote this fixed point by $SWin$, i.e.,
\begin{equation}
    SWin=SWin^n\quad s.t.\quad SWin^n=f_s(SWin^n)\text{.}
\end{equation}
By similar reasoning, we have
\begin{equation}
    EWin^{i+1}=f_e(EWin^i)=EWin^i\cap CPre^E(EWin^i)\text{,}
\end{equation}
and a fixed point $EWin$, which must exist,
\begin{equation}
    EWin=EWin^n\quad s.t.\quad EWin^n=f_e(EWin^n)\text{.}
\end{equation}

Computing $SWin$ and $EWin$ requires linear time in the number of states in $\G$. After a linear number of iterations at most, the fixed points can be reached. Furthermore, the lemma and proof presented below indicate that the iterative processes of $SWin^i$ and $EWin^i$ are synchronized and eventually reach their fixed points at the same time.
\begin{lemma}\label{lem:win-union}
    For $i\in\mathbb{N}$, $SWin^i\cup EWin^i=S$ holds.
\end{lemma}

\begin{proof}

We perform the lemma by induction on $n\in\mathbb{N}$.

\textbf{Basis.} Obviously we have $SWin^0\cup EWin^0=\emptyset\cup S=S$.

\textbf{Induction.} The induction hypothesis is that $SWin^i\cup EWin^i=S$.

For an arbitrary set of states $\D$, it is evident that:
\begin{equation}
\begin{aligned}
CPre^S(S\setminus\D)
    &=\{s\in S\mid\exists Y\in2^\Y.\forall X\in2^\X.(s,X\cup Y)\in T\text{ or }\delta(s,X\cup Y)\in S\setminus\D\}\\
    &=\{s\in S\mid\exists Y\in2^\Y.\forall X\in2^\X.(s,X\cup Y)\in T\text{ or }\delta(s,X\cup Y)\notin \D\}\\
    &=S\setminus\{s\in S\mid\forall Y\in2^\Y.\exists X\notin2^\X.(s,X\cup Y)\notin T\text{ and }\delta(s,X\cup Y)\in\D\}\\
    &=S\setminus CPre^E(\D)\text{.}\label{eq:cpre-minus}
\end{aligned}
\end{equation}
When $\D$ in Equation~(\ref{eq:cpre-minus}) is set equal to $EWin_i$, we derive Equation~(\ref{eq:substitute}). Then we further perform the following set operations, which may be more comprehensible if read in reverse order.
\begin{align}
    &CPre^S(S\setminus EWin^i)=S\setminus CPre^E(EWin^i)\label{eq:substitute}\\
\Rightarrow\quad&CPre^S(S\setminus EWin^i)\cap EWin^i=EWin^i\cap(S\setminus CPre^E(EWin^i))\\
\xLeftrightarrow{A\setminus B=A\cap(S\setminus B)}\quad&CPre^S(S\setminus EWin^i)\setminus (S\setminus EWin^i)=EWin^i\setminus CPre^E(EWin^i)\\
\xLeftrightarrow{\text{induction hypothesis}}\quad&CPre^S(SWin^i)\setminus SWin^i=EWin^i\setminus CPre^E(EWin^i)\label{eq:equ-change}
\end{align}
In Equation~(\ref{eq:equ-change}), the left side of the equal sign represents the increasing part from $SWin^i$ to $SWin^{i+1}$, while the right side corresponds to the decreasing part from set $EWin^i$ to $EWin^{i+1}$.
\begin{equation}
    SWin^{i+1}\setminus SWin^i=CPre^S(SWin^i)\setminus SWin^i
\end{equation}
\begin{equation}
    EWin^i\setminus EWin^{i+1}=EWin^i\setminus CPre^E(EWin^i)
\end{equation}
Thus we can conclude that $SWin^{i+1}\cup EWin^{i+1}=S$.
\end{proof}
With Lemmas~\ref{lem:win-insertion}, \ref{lem:win-union}, and the following theorem, we conclude that \tdfa games can be determined with Equations~(\ref{eq:cpres})-(\ref{eq:ewin}), which provide a complete and sound solution.
\begin{theorem}\label{thm:tdfa-game}
Given a \tdfa game $\G=( 2^{\X\cup\Y},S,\delta,T)_{Moore}$, $SWin$ is the set of system-winning states, and $EWin$ is the set of environment-winning states. Both players have a uniform and positional strategy.
\end{theorem}
\begin{proof}
We fix an arbitrary total ordering on $2^\X$ and $2^\Y$ respectively. Then we perform the proof with a case analysis. 

\paragraph{$\bm{s\in SWin}$.}
For state $s\in SWin$, we denote by $l(s)=min\{n\in\mathbb{N}\mid s\in SWin^n\}$ the level where $s$ is added to $SWin$ firstly. We have the following uniform and positional strategy $\pi:S\to2^\Y$.
\begin{equation}
\pi(s)=
\begin{cases}
    min\{Y\in 2^\Y\mid \forall X\in2^\X.(s,X\cup Y)\in T\text{ or }\delta(s,X\cup Y)\in SWin^{l(s)-1}\}&\text{ if }s\in SWin\\
    min(2^\Y)&\text{ if }s\in S\setminus SWin
\end{cases}
\end{equation}
By the induction on $n\in\mathbb{N}$, we can prove that any play that starts in $s\in SWin^n$ and is consistent with $\pi$ reaches accepting transitions $T$ within at most $n$ step.

\paragraph{$\bm{s\in EWin}$.}
For state $s\in S$, we denote by $h(s)=max\{n\in\mathbb{N}\mid s\in EWin^n\}$ the level where $s$ is not removed from $EWin$ lastly. We have the following uniform and positional strategy $\tau:S\times2^\Y\to2^\X$.
\begin{equation}
\tau(s,Y)=
\begin{cases}
    min\{X\in 2^\X\mid(s,X\cup Y)\notin T\text{ and }\delta(s,X\cup Y)\in EWin^{h(s)-1}\}&\text{ if }s\in EWin^1\\
    min(2^\X)&\text{ if }s\in S\setminus EWin^1
\end{cases}
\end{equation}

By the induction on $n\in\mathbb{N}$, we can prove that any plays $p$ that starts in $s\in EWin^n$ and consistent with $\tau$ cannot reach accepting transitions $T$ within $n$ steps.

\end{proof}

\begin{corollary}
    Given a \tdfa game $\G=(2^{\X\cup\Y},S,\delta,T)_{Moore}$, $s\in S$ is an environment-winning state iff $s$ is not a system-winning state.
\end{corollary}

\subsection{How to Reduce \LTLf Synthesis to \tdfa Games}

In a \TDFA derived from an \LTLf formula, a state being system-winning or environment-winning reflects the realizability of the corresponding \LTLf specification. The following lemma explicitly reveals the relationship between \TDFA games and \LTLf synthesis.

\begin{lemma}\label{lem:syn2game}
Given an \LTLf formula $\varphi$ and the corresponding \tdfa ${\A_{\varphi}}=\roundBra{2^\PP, S, \delta, init, T}$ constructed by Definition~\ref{def:ltlf2tdfa}, $\roundBra{\psi,\X,\Y}_{Moore}$ is realizable iff $\psi$ is a system-winning state in the game $\G=(2^{\X\cup\Y},S,\delta,T)_{Moore}$, where $\psi\in S$ is a state.
\end{lemma}

\begin{proof}

The proof of this lemma is performed with the following equivalence relations.

\textcolor{white}{$\Leftrightarrow$ }$\roundBra{\psi,\X,\Y}_{Moore}$ is realizable.

$\Leftrightarrow$ There exists a function $g: (2^\X)^* \to 2^\Y$ such that for any infinite sequence $\lambda = X_0, X_1, \cdots \in (2^\X)^\omega$, there exists $k \geq 0$ such that $\rho\models\psi$ holds, where $\rho=(X_0\cup g(\epsilon)),(X_1\cup g(X_0)),\cdots,(X_k\cup g(X_0,\cdots,X_{k-1}))$.

$\Leftrightarrow$
There exists a function $g: (2^\X)^* \to 2^\Y$ such that for any infinite sequence $\lambda = X_0, X_1, \cdots \in (2^\X)^\omega$, there exists $k \geq 0$ such that $r=s_0,\cdots,s_k$ is an accepting state sequence starting from $\psi$ and corresponding to $\rho=(X_0\cup g(\epsilon)),(X_1\cup g(X_0)),\cdots,(X_k\cup g(X_0,\cdots,X_{k-1}))$. Specifically, for the state sequence $r$ we have $s_0=\psi$, $s_i=\delta(s_{i-1},\rho[i-1])$ for $i=1\cdots,k$, and $(s_k,\rho[k])\in T$.

$\Leftrightarrow$
There exists a function $\pi: S^+\to2^\Y$ such that for any infinite sequence $\lambda = X_0, X_1, \cdots \in (2^\X)^\omega$, there exists $k \geq 0$ such that $p = (s_0, X_0\cup\pi(s_0)), \cdots,(s_k,X_k\cup\pi(s_0,\cdots,s_k))$ is a system-winning play starting from $\psi$ (i.e., $s_0=\psi$).

$\Leftrightarrow$
$\psi$ is a system-winning state in the game $\G$.

The constructive proof of the equivalence relation in the third step is as follows.

$(\Rightarrow)$ With the existence of the function $g:(2^\X)^*\to2^\Y$ as in the definition of realizability, we construct a system-winning strategy $\pi:S^+\to2^\Y$ from $\psi$ as follows. We first define the set of state sequences that are starting from $\psi$ and induced by $g$.
\begin{equation}
\begin{aligned}
    \R(\psi,g)=
    \{r\in S^+\mid&\text{\circled{1}}r[0]=\psi\text{ and}\\
    &\text{\circled{2}}\exists\eta\in(2^\X)^{|r|-1}. r[i]=\delta(r[i-1],\eta[i-1]\cup g(\eta^{i-1})),i=1,\cdots,|r|-1\}
\end{aligned}
\end{equation}
We fix an arbitrary total ordering on $(2^\X)^*$. For $r\in S^+$, we have
\begin{equation}
    \pi(r)=
    \begin{cases}
        g(f(r))&\text{ if }r\in\R(\psi,g)\\
        \emptyset&\text{ if }r\notin\R(\psi,g)
    \end{cases}\text{,}
\end{equation}
where the function $f:S^+\to(2^\X)^*$ maps state sequences to sequences of  propositional interpretations over $\X$.
\begin{equation}
    f(r)=
    \begin{cases}
        \begin{aligned}
        min\{&\eta\in(2^\X)^{|r|-1}\mid\\
        &r[i]=\delta(r[i-1]\text{ and }\eta[i-1]\cup g(\eta^{i-1})),i=1,\cdots,|r|-1\}
        \end{aligned}&\text{ if }r\in\R(\psi,g)\\
        min((2^\X)^*)&\text{ if }r\notin\R(\psi,g)
    \end{cases}
\end{equation}

$(\Leftarrow)$ With the existence of the system-winning strategy from $\psi$ $\pi:S^+\to2^\Y$, we construct the function $g:(2^\X)^*\to2^\Y$ that satisfies in the definition of realizability. For $\eta\in(2^\X)^*$, we have 
\begin{equation}
    g(\eta)=\pi(h(\eta))\text{,}
\end{equation}
where the function $h:(2^\X)^*\to S^+$ maps sequences of propositional interpretations over $\X$ to state sequences.
\begin{equation}
\begin{aligned}
    &h(\eta)=r\\
    s.t.\quad &r\in(S)^{|\eta|+1}\text{, }r[0]=\psi\text{, and }r[i]=\delta(r[i-1],\eta[i-1]\cup\pi(r^i))\text{ for }i=1,\cdots,|r|-1
\end{aligned}
\end{equation}
\end{proof}

The following theorem can be directly derived from Lemma~\ref{lem:syn2game}. When using \TDFA games to address the \LTLf realizability problem, the objective for reasoning \TDFA games is to determine whether the initial state $init$ is a system-winning state.

\begin{theorem}\label{thm:syn-game}
Given an \LTLf formula $\varphi$ and the corresponding \tdfa ${\A_{\varphi}}=(2^\PP, S, \delta, init, T)$ constructed by Definition~\ref{def:ltlf2tdfa}, the \LTLf specification $\roundBra{\varphi,\X,\Y}_{Moore}$ is realizable iff the initial state $init$ is a system-winning state in the game $\G=(2^{\X\cup\Y},S,\delta,T)_{Moore}$.
\end{theorem}

\section{On-the-fly \ltlf Synthesis}\label{sec:on-the-fly}

\ltlf synthesis can be decomposed into two tasks: the construction of a \tdfa with doubly exponential complexity, and the resolution of the \TDFA game in linear time. As detailed in Section~\ref{sec:tdfa}, we can convert an \ltlf formula to an equivalent \tdfa through a top-down method, where the intermediate results comprise the final automata directly. This enables us to address both tasks concurrently. In this section, we propose an on-the-fly approach designed to bypass the complete construction of the automaton and its associated doubly exponential complexity. This method conducts a \emph{global forward} depth-first search along with the \TDFA construction. During this process, we dynamically identify system-winning and environment-winning states based on the available information. As the process of depth-first searching progresses, we partition the states into their respective {\SCC}s. Upon identifying each \SCC, we perform a \emph{local backward} fixed-point search to determine the states left unresolved by the forward search.

\subsection{Approach Overview}

\begin{figure}[t]
  \centering
  \includegraphics[width=.99\linewidth]{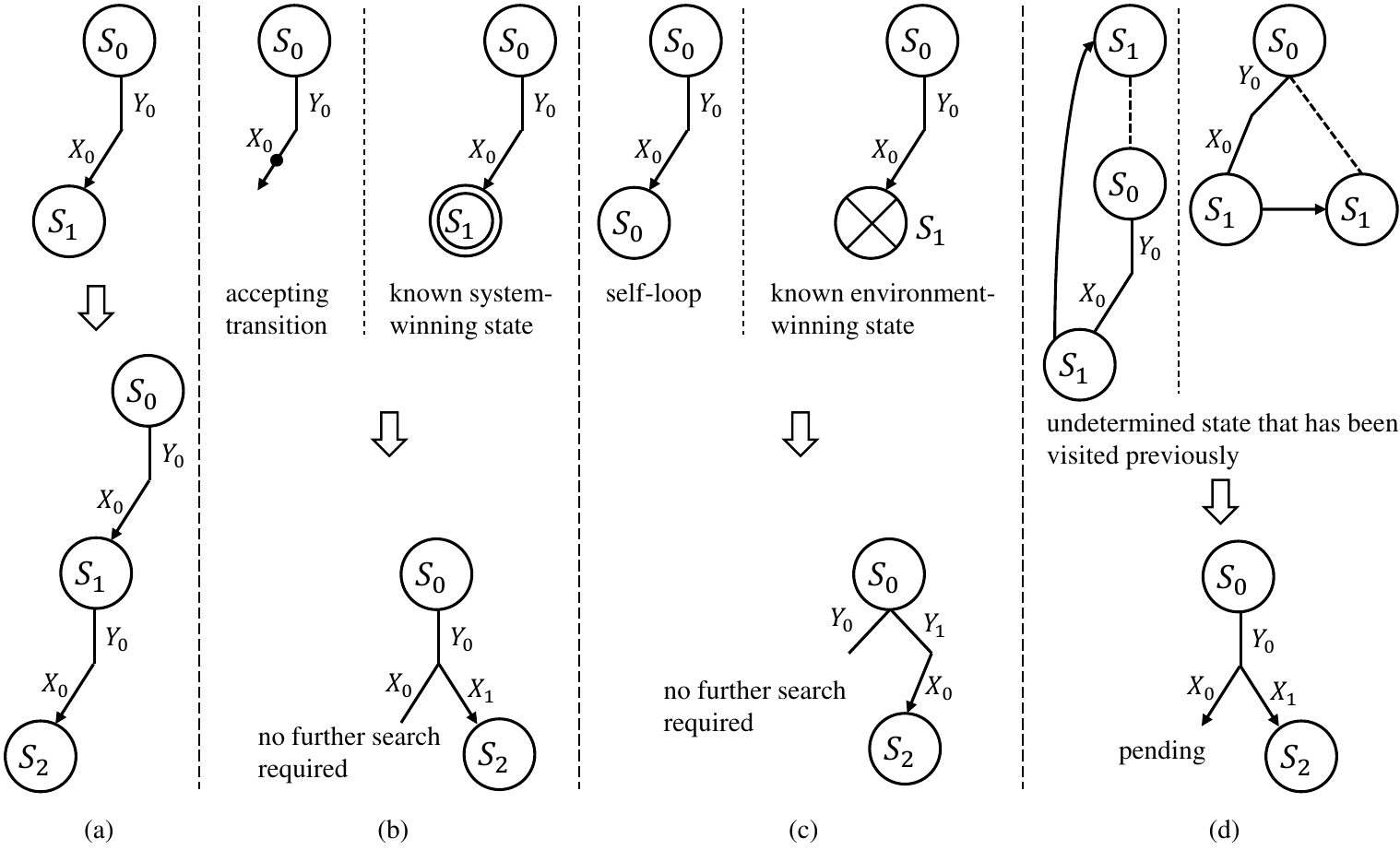}
  \caption{Different scenarios encountered during the forward search process and corresponding operations.}
  \label{fig:approach-overview}
  \Description{Demonstration of different scenarios encountered during the forward search process and corresponding operations.}
\end{figure}

Recall Theorem~\ref{thm:syn-game}, solving the problem of \LTLf realizability requires determining the corresponding \TDFA games, i.e., determining the winning states for the system/environment in games. Intuitively, to identify a state as system-winning, one needs to find a system choice such that, regardless of the response of the environment, the system can enforce a win. Conversely, to determine an environment-winning state, for every action of the system, the environment must be able to find a winning response.

Given an \ltlf specification $(\varphi,\X,\Y)_{Moore}$, we depart from the initial state $\varphi$ and utilize formula progression to incrementally build the automaton $\A_\varphi$ with a depth-first strategy. Figure~\ref{fig:approach-overview} demonstrates the different scenarios encountered during the forward search process and corresponding operations employed to determine the \TDFA game. Assume that our current state is $s_0$ and we have an edge $X_0\cup Y_0$.

\paragraph{\textup{(a)}} If $(s_0,X_0\cup Y_0)\notin T$ is not an accepting transition and $\delta(s_0,X_0\cup Y_0)=s_1$ is an unexplored state, there is no new useful information at the moment. Thus the depth-first strategy is applied to generate a new state $s_2$.

\paragraph{\textup{(b)}} If $(s_0,X_0\cup Y_0)\in T$ is an accepting transition or $\delta(s_0,X_0\cup Y_0)=s_1$ is a known system-winning state, the combined system and environment choice $X_0\cup Y_0$ can help the system win. And further exploration for $(s_0,X_0\cup Y_0)$ is unnecessary. Then the search shifts to exploring an alternative environment choice.

\paragraph{\textup{(c)}} If \circled{1} $(s_0,X_0\cup Y_0)\notin T$ is not an accepting transition, and \circled{2} $\delta(s_0,X_0\cup Y_0)=s_0$ forms a self-loop or $\delta(s_0,X_0\cup Y_0)=s_1$ is a known environment-winning state, the environment can enforce to win with current system choice $Y_0$. Thus further exploration for environment choices with $(s_0,Y_0)$ is unnecessary. The process then proceeds by considering a different system choice.

\paragraph{\textup{(d)}} If $(s_0,X_0\cup Y_0)$ is not an accepting transition and $\delta(s_0,\X_0\cup Y_0)=s_1$ is a previously visited state that has not been identified as system/environment-winning, no new useful information are gained. The repeatedly encountered state $s_1$ can be an ancestor or descendant of $s_0$ before this occurrence. For the former case, $s_1$ has not been fully processed in the depth-first search and is thus undetermined now; for the latter case, the depth-first search concerning $s_1$ is complete. Regardless of the case, the decision regarding $(s_0,X_0\cup Y_0)$ is pending and no repeat exploration is undertaken. The search then moves forward with other edges.

In scenarios (b) and (c) above, transitions leading to other known system/environment-winning states may occur either immediately upon reaching the successor state or at the backtracking phase of the depth-first search. Additionally, the case of (d) introduces the possibility that, a state $s$ has been fully processed in the depth-first search but has still not been determined. That a state $s$ is fully processed in the depth-first search indicates that all successors of $s$ have been explored following the depth-first strategy. It implies that the search does not backtrack to $s$ anymore. To this end, we integrate the detection of {\SCC}s within the forward depth-first search. Upon finding a new \SCC, we conduct a local backward search concerning Equations (\ref{eq:swin}) and (\ref{eq:ewin}) within the range of the \SCC just found.

\subsection{Algorithm}

Algorithm~\ref{alg:main-syn} shows the details of our on-the-fly \ltlf synthesis approach. It takes an \ltlf specification $(\varphi,\X,\Y)_{Moore}$ as the input and returns the realizability of the specification. The algorithm first declares five \emph{global} sets:
\begin{itemize}
\item $swin\_state$ and $ewin\_state$ to collect the known system/environment-winning states respectively;

\item $undetermined\_state$ to collect the states that have been visited but not determined to be system/environment-winning;

\item $swin\_transition$ to store system-winning transitions $(s,X\cup Y)$ such that $X\cup Y\models s$ holds or $\fp{s,X\cup Y}$ is a known system-winning state;

\item $ewin\_transition$ to store environment-winning transitions $(s,X\cup Y)$ such that \circled{1}$X\cup Y\not\models s$ holds, and \circled{2}$\fp{s,X\cup Y}$ is $s$ itself or a known environment state;

\item $ewin\_transition$ to store undetermined transitions $(s,X\cup Y)$ such that $X\cup Y\not\models s$ holds and $\fp{s,X\cup Y}$ is a state that has been visited but not determined.
\end{itemize}
Then the main procedure $\forwardSearch(\varphi)$ is invoked to compute and determine some (potentially \emph{all} in the worst case) system/environment-winning states. Finally, based on whether the initial state is a system-winning state, it returns `Realizable' or `Unrealizable' accordingly.

\begin{algorithm}[htbp!]
\caption{On-the-fly \ltlf Synthesis}\label{alg:main-syn}
\LinesNumbered
\DontPrintSemicolon
\KwIn{An \ltlf specification $(\varphi,\X,\Y)_{Moore}$}
\KwOut{Realizable or Unrealizable}

$swin\_state\coloneqq\emptyset$, $ewin\_state\coloneqq\emptyset$, $undetermined\_state\coloneqq\{\varphi\}$\;
$swin\_transition\coloneqq\emptyset$, $ewin\_transition\coloneqq\emptyset$, $undetermined\_transition\coloneqq\emptyset$\;
$\forwardSearch(\varphi)$\;
\KwRet $(\varphi\in swin)\,?\,$Realizable$\,:\,$Unrealizable\;
\;
\myproc{$\forwardSearch(s)$}
{
  \While{$\trueVal$}
  {\label{line:main:loop-in}
    \If{$\checkCurrentStatus(s)\neq\;$\textup{Unknown}}
    {\label{line:main:check-current}
      \Break\;
    }
    $edge\coloneqq \getEdge(s)$\;\label{line:main:get-edge}
    \If{$edge=\;$\textup{Null}}
    {
      \Break\;
    }
    \If{$\fp{s,edge}\in undetermined\_state$}
    {
      $undetermined\_transition.\insert(s,edge)$\;
      \Continue\;
    }
    $undetermined\_state.\insert(\fp{s,edge})$\;
    $\forwardSearch(\fp{s,edge})$\;\label{line:main:recursive}
  }\label{line:main:loop-out}
  \If{$\isSccRoot(s)$}
  {
    $scc\coloneqq\getScc()$\;
    $\backwardSearch(scc)$\;
  }
}
\end{algorithm}

The main procedure of our approach is $\forwardSearch(s)$. It initiates with some current state $s$, and recursively explores successors of $s$ that are necessary to be searched in the \emph{while}-loop at Lines \ref{line:main:loop-in}-\ref{line:main:loop-out}. Within each iteration, the procedure first tries to determine based on known information whether $s$ is system/environment-winning (Line~\ref{line:main:check-current}). If $s$ can be determined currently, it leaves the loop. At Line \ref{line:main:get-edge}, it attempts to calculate an edge comprising system and environment choice that requires further exploration. Specifically, $\getEdge()$ assigns $edge$ as:
\begin{equation}\label{eq:get-edge}
\begin{aligned}
    &X\cup Y\in 2^{\X\cup\Y}\\
    s.t.\quad &Y\notin Y_{known\_ewin}(s)\text{, }X\notin X_{known\_swin}(s,Y)\text{, and }X\notin X_{known\_undetermined}(s,Y)\text{.}
\end{aligned}
\end{equation}
Here, $Y_{known\_ewin}(s)$ denotes the set of known system choices, with which some environment choices can lead the environment to win from $s$.
\begin{equation}\label{eq:knownEwin}
    Y_{known\_ewin}(s) = \{Y\in2^\Y\mid \exists X\in2^\X. (s,X\cup Y)\in ewin\_transition\}
\end{equation}
For other system choices that are not in $Y_{known\_ewin}$, the system retains the potential for winning. Moreover, it no longer needs to explore environment choices that are known to lead the system winning, represented by $X_{known\_swin}(s,Y)$, or those leading to visited but undetermined states, represented by $X_{undetermined\_swin}(s,Y)$.
\begin{equation}\label{eq:knownSwin}
    X_{known\_swin}(s,Y) = \{X\in2^\X\mid(\psi,X\cup Y)\in swin\_transition\}
\end{equation}
\begin{equation}\label{eq:knownUndetermined}
    X_{known\_undetermined}(s,Y) = \{X\in2^\X\mid(\psi,X\cup Y)\in undetermined\_transition\}
\end{equation}
If there does not exist an edge that satisfies the condition in Equation~(\ref{eq:get-edge}), it will get `Null' from $\getEdge(s)$ and then leave the loop. This indicates that it finishes the exploration for the successors of $s$. Conversely, if an edge that satisfies the condition in Equation~(\ref{eq:get-edge}) is found, the algorithm then processes the successor $\fp{s,edge}$. If $\fp{s,edge}$ is a visited but undetermined state (i.e., $\fp{s,edge}\in undetermined\_state$), $(s,edge)$ is added to $undetermined\_transition$ and it continues the next iteration; otherwise,  proceeds to recursively explore the successor $\fp{s,edge}$ (Line~\ref{line:main:recursive}).
Finally, upon exiting the loop, the algorithm checks whether a \SCC has been detected (Line \ref{line:main:check-scc}). If detected, a backward search is conducted for the undetermined states within the \SCC (Line~\ref{line:main:bsearch-scc}).

\begin{algorithm}[htbp!]
\caption{Implementation of \texttt{checkCurrentStatus}}\label{alg:checkCurrentStatus}
\LinesNumbered
\DontPrintSemicolon
\KwIn{A state $s$ of \tdfa $\A_\varphi$}
\KwOut{System-winning, Environment-winning, Unknown}

\If{$\currentSystemWinning{s}$}
{
  \KwRet System-winning\;
}
\If{$\currentEnvironmentWinning{s}$}
{
  \KwRet Environment-winning\;
}
\KwRet Unknown\;
\;
\myproc{$\currentSystemWinning{s}$}
{
  \If{$s\in swin\_state$}
  {\label{line:checkCurrent:sInSwin}
    \KwRet $\trueVal$\;
  }
  \For{each $Y\in 2^{\Y}$}
  {
    $all\_Y\_swin\coloneqq\trueVal$\;
    \For{each $X\in 2^{\X}$}
    {
      \If{$X \cup Y\not\models s$ and $\fp{s, X\cup Y}\notin swin\_state$} 
      {\label{line:checkCurrent:fpNotinSwin}
        $all\_Y\_swin \coloneqq \falseVal$\;
        \Break\;
      }
      \Else
      {
        $swin\_transition.\insert(s,X\cup Y)$\;
      }
    }
    \If{$all\_Y\_swin$}
    {\label{line:main:check-scc}
      $swin\_state.\insert(s)$\;
      $undetermined\_state.\remove(s)$\;
      \KwRet $\trueVal$\;\label{line:main:bsearch-scc}
    }
  }
  \KwRet $\falseVal$\;
}
\myproc{$\currentEnvironmentWinning{s}$}
{
  \If{$s\in ewin\_state$}
  {
    \KwRet $\trueVal$\;
  }
  \For{each $Y\in 2^{\Y}$}
  {\label{line:checkCurrent:sInEwin}
    $exist\_X\_fail\coloneqq\falseVal$\;
    \For{each $X\in 2^{\X}$}
    {
      \If{$X\cup Y\not\models s$\ and\ $\fp{s, X\cup Y}\in (\{s\}\cup ewin\_state)$}
      {\label{line:checkCurrent:fpInEwin}
        $exist\_X\_fail \coloneqq\trueVal$\;
        $ewin\_transition.\insert(s,X\cup Y)$\;
        \Break\;
      }
    }
    \If{$\neg exist\_X\_fail$}
    {
      \KwRet $\falseVal$\;
    }
  }
  $ewin\_state.\insert(s)$\;
  $undetermined\_state.\remove(s)$\;
  \KwRet $\trueVal$\;
}

\end{algorithm}

The implementation of $\checkCurrentStatus{s}$ is presented in Algorithm~\ref{alg:checkCurrentStatus}. It checks whether $s$ is system/environment-winning currently based on the state information collected so far. The procedure $\checkCurrentStatus{s}$ returns `System-winning' if \circled{1} $s$ is already in $swin$, or \circled{2} there exists $Y\in2^\Y$ such that $X\cup Y\models s$ holds or $\fp{s,X\cup Y}$ is in $swin\_state$, for every $X\in2^\X$. During the process, system-winning states and transitions newly found are added into $swin\_state$ and $swin\_transition$ respectively, new system-winning states are removed from $undetermined\_state$. The analogous process is performed to check whether $s$ is an environment-winning state currently. Besides, it is feasible to couple \currentSystemWinning, \currentEnvironmentWinning, and \getEdge within a two-level nested loop for efficiency. The current presentation in the paper aims to enhance readability.

Subsequently, Algorithm~\ref{alg:backwardSearch} exhibits the implementation of $\backwardSearch{C}$, where it calculates the local fixed point within the range of $C$. From the set $cur\_swin$ of the current system-winning states, it first finds their direct predecessors that are in $C$ and undetermined (Line~\ref{line:bsearch:candidate}), which are potentially be determined as system-winning and stored in $candidate\_new\_swin$. Then it tries to determine whether states in $candidate\_new\_swin$ are system-winning. And states that are found system-winning are collected in $new\_win$, which is used as the current system-winning states in the next round of iteration. When no new system-winning state is found in a round iteration, it has found all the system-winning states in $C$ and leaves the loop. Ultimately, states in $C$ that are not system-winning are determined as environment-winning states and added into $ewin\_state$. 

\begin{algorithm}[htbp!]
\caption{Implementation of \texttt{backwardSearch}}\label{alg:backwardSearch}
\LinesNumbered
\DontPrintSemicolon
\KwIn{A strongly connected component $C$}

$cur\_swin\coloneqq swin\cap C$\;
\Do{$new\_swin\neq\emptyset$}
{
  $new\_swin\coloneqq\emptyset$\;
  $candidate\_new\_swin\coloneqq \directPredecessors(cur\_win,C)\cap undetermined\_state$\;\label{line:bsearch:candidate}%tcp*[l]{\predecessor(cur\_win) returns the set of direct predecessors in $C$ of states in $cur\_win$.}
  \For{$s\in candidate\_new\_swin$}
  {
    \If{$\currentSystemWinning{s}$}
    {
      $new\_swin.\insert(s)$\;
    }
  }
  $cur\_swin\coloneqq new\_swin$\;
  $swin\_state\coloneqq swin\_state\cup new\_swin$\;
  $undetermined\_state\coloneqq undetermined\_state\setminus new\_swin$\;
}
$ewin\_state\coloneqq ewin\_state\cup (C\setminus swin\_state)$\;
$undetermined\_state\coloneqq undetermined\_state\setminus C$\;
\end{algorithm}

\subsection{Theoretical Analysis of the Algorithm}

TIn the following, we analyze and establish the correctness and complexity of the algorithm. We first define a set $Y_{swin}(s)$ that is used in the subsequent proof. Intuitively, for some a state $s$, it represents the set of system choices with which the state can be system-winning.
\begin{equation}
    Y_{swin}(s)=\{Y\in2^\Y\mid \forall X\in2^\X.\text{\circled{1}}(s,X\cup Y)\in T \text{, or \circled{2}}\delta(s,X\cup Y)\neq s\text{ and }\delta(s,X\cup Y)\in SWin\}
\end{equation}
The following two lemmas demonstrate the correctness of determining winning states for the system/environment in forward search.

\begin{lemma}\label{lem:forward-swin}
Given a \tdfa game $\G=(2^{\X\cup\Y},S,\delta,T)_{Moore}$, $s\in S$ is a system-winning state iff there exists $Y\in2^\Y$ such that, for every $X\in2^\X$, either $(s,X\cup Y)\in T$ is an accepting transition or the successor $\delta(s,X\cup Y)$ is another system-winning state.
\end{lemma}
\begin{proof}

Intuitively, this lemma can be directly derived from the construction process of $SWin^i$ (i.e., Equation~(\ref{eq:swin})). Recall that for $t\in SWin$, we denote by $l(t)$ the level where $t$ is added to $SWin$ firstly. The process of constructing $l(s)$ below demonstrates that $s$ can definitely be added to $SWin$, implying that $s$ is a system-winning state. % Notice that $Y_{swin}(s)$ corresponds exactly to the condition stated in the lemma.

For every winning system choice of $s$, we can find the maximum level at which its system-winning successors are added to $SWin$ for the first time. For $Y\in Y_{swin}$, we define
\begin{equation}
    m(Y)=
    \begin{cases}
        0&\text{if }\forall X\in2^\X.(s,X\cup Y)\in T\text{ holds}\\
        max\{l(\delta(s,Y\cup X))\mid X\in2^\X\text{, }(s,X\cup Y)\notin T\}&\text{otherwise}
    \end{cases}\text{.}
\end{equation}
Then we have $l(s)=1+min\{m(Y)\mid Y\in Y_{swin}(s)\}$.

\end{proof}

\begin{lemma}\label{lem:forward-ewin}
Given a \tdfa game $\G=(2^{\X\cup\Y},S,\delta,T)_{Moore}$, $s\in S$ is an environment-winning state if for every $Y\in2^\Y$, there exists $X\in2^\X$ such that, $(s\times X\cup Y)\notin T$ is not an accepting transition, and the successor $\delta(s,X\cup Y)$ is either itself (i.e., $\delta(s,X\cup Y)=s$) or another environment-winning state.
\end{lemma}
\begin{proof}

Assume that $s$ is a system-winning state. By Theorem~\ref{thm:tdfa-game}, there exists a positional and uniform system winning strategy $\pi:S\to2^\Y$ for $\G$. For $\pi(s)$, there exists $X\in2^\X$ such that, $(s\times X\cup\pi(s))\notin T$ is not an accepting transition, and the successor $\delta(s,X\cup\pi(s))$ is either itself (i.e., $\delta(s,X\cup\pi(s))=s$) or another environment-winning state.
\begin{itemize}
\item $\delta(s,X\cup\pi(s))=s$. Then there exists an environment-winning play $(s,X\cup\pi(s))^\omega$ that is consistent with $pi$ and starting from $s$, which contradicts the assumption that $s$ is a system-winning state.

\item $\delta(s,X\cup\pi(s))=t$ with $t\neq s$ is another environment-winning state. Then there exists a play $p$ that is consistent with $pi$ and starting from $t$. The play $(s,X\cup\pi(s)),p$ is also an environment-winning play that is consistent with $pi$, which contradicts the assumption.
\end{itemize}
The assumption that $s$ is a system-winning state leads to a contradiction, so we have that $s$ is an environment-winning state.
\end{proof}

As for the backward search phase, we aim to establish that, all the essential information related to each \SCC has been collected before initiating the backward search. This enables us to determine all previously undecided states within the \SCC through the backward search process. To facilitate this, we first introduce the concept of minimal dependent structure.
\begin{definition}[Minimal Dependent Structure of System-winning State]\label{def:MDS}
Given a \tdfa game $\G=(2^{\X\cup\Y},S,\delta,T)_{Moore}$, a minimal dependent structure $U_s$ of a system-winning state $s\in SWin$ is a subset of $S\times2^\Y$ such that:
\begin{equation}
    U_s=
    \begin{cases}
        \{(s,Y)\}&\text{ if }sna(s,Y)=\emptyset\\
        \{(s,Y)\}\cup \bigcup_{t\in sna(s,Y)}U_t&\text{ if }sna(s,Y)\neq\emptyset
    \end{cases}\text{,}
\end{equation}
where $Y\in Y_{swin}(s)$ hold and $U_t$ is a minimal dependent structure of $t$. And $sna(t,Y)$ represents the set of \underline{s}uccessors of $t$ that correspond to the system choice $Y$ and are reached via \underline{n}on-\underline{a}ccepting transitions, defined as $sna(t,Y)=\{v\in S\mid v=\delta(t,X\cup Y)\text{, }(t,X\cup Y)\notin T\text{, and }X\in2^\X\}$.
\end{definition}
Intuitively, a `minimal dependent structure' of a system-winning state encapsulates all the essential information required to classify it as system-winning. The usage of `Minimal' indicates that no subset of this structure suffices for such determination. Moreover, since there can be multiple system choices that enable a state to be system-winning, a system-winning state may correspond to several distinct minimal dependent structures. This concept serves merely as a tool to demonstrate the soundness of our approach; hence, we do not delve into further discussion of its properties.

\begin{lemma}\label{lem:MDS-generate-G}
Given a \tdfa game $\G=(2^{\X\cup\Y},S,\delta,T)_{Moore}$, a system-winning state $s$ in $\G$, and a minimal dependent structure $U_s$ of $s$, $s$ is also a system-winning state in $G'$, where $\G'=(2^{\X\cup\Y},S'\cup\{p\},\delta',T')_{Moore}$ is a \tdfa game such that:
\begin{itemize}
    \item $U_s|_S\subseteq S'\subseteq S$ holds, $p$ is a padding state, and $U_s|_S=\{t\in S\mid \exists Y\in2^\Y.(t,y)\in U_s\}$;

    \item $\delta'(t,X\cup Y)=
        \begin{cases}
        \delta(t,X\cup Y)&(t,Y)\in U_s\text{ and }\delta(t,X\cup Y)\in sna(t,Y)\\
        p& t=p\text{ or }\delta(t,X\cup Y)\notin S'\\
        \delta(t,X\cup Y)\text{ or }p&\text{otherwise}
        \end{cases}$;

    \item $T'=T\cap(S'\times2^{\X\cup\Y})$ holds.
\end{itemize}
\end{lemma}
\begin{proof}
With Lemma~\ref{lem:forward-swin}, we can prove that states in $U_s|_S$ are system-winning states in $G'$ by the structural induction on Definition~\ref{def:MDS}. This establishes that $s$ is a system-winning state in $G'$.
\end{proof}

\begin{lemma}\label{lem:MDS-subG}
Algorithm~\ref{alg:backwardSearch} correctly identifies the undetermined states within the current \SCC $C$ as system/environment-winning.
\end{lemma}
\begin{proof}
States in $undetermined\_state$ before executing Algorithm~\ref{alg:backwardSearch} can result in two possible situations after the execution. Firstly, a state $s$ is determined to be system-winning and added to $swin$, with the correctness of this case guaranteed by Lemma~\ref{lem:forward-swin}. Secondly, if $s$ is not identified as system-winning, it is classified as environment-winning and added to $ewin$. To validate this situation, we conduct a proof by contradiction: assume there exists a state $s$, not identified as system-winning and yet is in fact system-winning. Considering our depth-first search strategy and the properties of Tarjan's algorithm, {\SCC}s form a directed acyclic graph and are traversed in depth-first and topological order.
This ensures the existence of a minimal dependent structure for $s$, which satisfies the conditions of Lemma~\ref{lem:MDS-generate-G} in conjunction with the currently constructed partial \TDFA game. Consequently, $s$ would have been identified as system-winning during the backward search in Algorithm~\ref{alg:backwardSearch}, contradicting the initial assumption. Hence, all state $s$ not recognized as system-winning by Algorithm~\ref{alg:backwardSearch} is correctly classified as environment-winning.
\end{proof}

\begin{theorem}
Given an \ltlf specification $(\varphi,\X,\Y)_{Moore}$,
\begin{enumerate}
\item Algorithm \ref{alg:main-syn} can terminate within time  of $O(2^{|\X\cup\Y|}\cdot 2^{2^{|tcl(\varphi)|}})$;

\item $(\varphi,\X,\Y)_{Moore}$ is realizable iff Algorithm \ref{alg:main-syn} returns `\textup{Realizable}'.
\end{enumerate}
\end{theorem}
\begin{proof}
By Theorem~\ref{thm:tdfaComplexity}, maximum of $2^{2^{|tcl(\varphi)|}}$ states is visited. And each state is limited to a backtrack count of at most $2^{|\X\cup\Y|}$. Thus the time complexity of the algorithm is $O(2^{|\X\cup\Y|}\cdot 2^{2^{|tcl(\varphi)|}})$.

By Theorem~\ref{thm:tdfa-game}$, (\varphi,\X,\Y)_{Moore}$ is realizable iff $\varphi$ is a system-winning state of the corresponding \tdfa game. And the correctness of determining system/environment winning states has been discussed in Lemmas \ref{lem:forward-swin}, \ref{lem:forward-ewin}, and \ref{lem:MDS-subG}, covering all possible scenarios.

\end{proof}

\section{Optimization Techniques}\label{sec:opt}

In this section, we extend Algorithm~\ref{alg:main-syn} by presenting two optimization techniques from different perspectives.

\subsection{Model-guided Synthesis}

Our on-the-fly synthesis approach requires determining whether the states in the corresponding \tdfa game are winning states for the system or the environment. In this process, the edges of each state are explored in a random and non-directed manner. By Lemma~\ref{lem:forward-swin}, system-winning states are recursively detected with its base case falling on the accepting edges of \tdfa. This leads us to intuitively infer that edges associated with some satisfiable traces are more likely to result in the current state being identified as system-winning. Inspired by this insight, we design the model-guided strategy to select the proceeding directions, which achieves a more targeted search. Here, models refer to satisfiable traces.

The \LTLf satisfiability problem has been addressed by a relatively efficient solution~\cite{LRPZV19}, which provides two {\API}s for our use.
\begin{itemize}
\item $\ltlfSat(\varphi)$ checks whether an \LTLf formula $\varphi$ is satisfiable;

\item $\getModel()$, invoked when $\ltlfSat(\varphi)$ returns `sat', retrieves a model of $\varphi$, which is a satisfiable trace of minimum length.
\end{itemize}

\begin{algorithm}[htbp!]
\caption{Model-guided \ltlf Synthesis}\label{alg:mg-syn}
\LinesNumbered
\DontPrintSemicolon
\KwIn{An \ltlf specification $(\varphi,\X,\Y)_{Moore}$}
\KwOut{Realizable or Unrealizable}
$swin\_state\coloneqq\emptyset$, $ewin\_state\coloneqq\emptyset$, $undetermined\_state\coloneqq\{\varphi\}$\;
$swin\_transition\coloneqq\emptyset$, $ewin\_transition\coloneqq\emptyset$, $undetermined\_transition\coloneqq\emptyset$\;
$model\coloneqq\epsilon$\;
$\forwardSearch(\varphi)$\;
\KwRet $(\varphi\in swin)\,?\,$Realizable$\,:\,$Unrealizable\;
\;
\myproc{$\forwardSearch(s)$}
{
  \While{$\trueVal$}
  {
    \If{$\checkCurrentStatus(s)\neq\;$\textup{Unknown}}
    {
      \Break\;
    }
    $edge\coloneqq \getEdge(s)$\;\label{line:mg-syn:getEdge}
    \If{$edge=\;$\textup{Null}}
    {\label{line:mg-syn:nullEdge}
      \If{$\noSwinPotential(s)$}
      {
        $ewin\_state.\insert(s)$\;
        $undetermined\_state.\remove(s)$\;
      }
      \Break\;
    }
    \If{$\fp{s,edge}\in undetermined\_state$}
    {
      $undetermined\_transition.\insert((s,edge))$\;
      $model\coloneqq\epsilon$\;\label{line:mg-syn:resetRevisit}
      \Continue\;
    }
    $undetermined\_state.\insert(\fp{s,edge})$\;
    $\forwardSearch(\fp{s,edge})$\;\label{line:mg-syn:resetBacktrack}
    $model\coloneqq\epsilon$\;
  }
  \If{$\isSccRoot(s)$}
  {
    $C\coloneqq\getScc()$\;
    $\backwardSearch(scc)$\;
  }
}
\myproc{$\getEdge(s)$}
{\label{line:mg-getEdge:begin}
  \If{$model=\epsilon$}
  {\label{line:mg-getEdge:checkEmpty}
    $edge\_constraint\coloneqq\edgeConstraint(s)$\;\label{line:mg-getEdge:edgeConstraint}
    \If{$\ltlfSat(s\wedge edge\_constraint)=\;$\textup{sat}}
    {
      $model\coloneqq\getModel()$\;
    }
  }
  \If{$model=\epsilon$}
  {
    \KwRet Null\;
  }
  \Else
  {
    $edge\coloneqq model[0]$\;
    $model\coloneqq model_1$\quad \tcp*[l]{remove $model[0]$}
    \KwRet $edge$\;
  }
}\label{line:mg-getEdge:end}
\end{algorithm}

Algorithm~\ref{alg:mg-syn} shows the implementation of the model-guided approach, which is built upon Algorithm~\ref{alg:main-syn}. We now proceed to clarify their differences. Firstly, Algorithm~\ref{alg:mg-syn} introduces a new global variable $model$, which stores a satisfiable trace and is initialized as empty trace $\epsilon$. Secondly, it selects edges for search by invoking the $\getEdge(s)$ at Line~\ref{line:mg-syn:getEdge}, which adopts the model-guided strategy to choose paths forward.

The implementation of $\getEdge(s)$ in model-guided synthesis is detailed at Lines~\ref{line:mg-getEdge:begin}-\ref{line:mg-getEdge:end}. It first checks whether $model$ is an empty trace (Line~\ref{line:mg-getEdge:checkEmpty}). If $model$ is not an empty trace, this indicates that some unexplored segment exists within a satisfiable trace previously computed. In the case where $model$ is now empty, it tries to acquire a new satisfiable trace. At Line~\ref{line:mg-getEdge:edgeConstraint}, an edge constraint is computed for current $s$ to block edges that do not require further exploration. Formally, $\edgeConstraint(s)$ assigns $edge\_constraint$ as:
\begin{equation}
    \bigwedge_{Y\in Y_{known\_ewin}(s)}\neg Y
    \wedge
    \bigwedge_{Y\notin Y_{known\_ewin}(s)}\roundBra{Y\to\bigwedge_{X\in X_{known\_swin}(s,Y)\cup X_{known\_undermined}(s,Y)}\neg X}\text{,}
\end{equation}
where $Y_{known\_ewin}(s)$, $X_{known\_swin}(s,Y)$, and $X_{known\_undermined}(s,Y)$ are defined as Equations (\ref{eq:knownEwin}), (\ref{eq:knownSwin}), and (\ref{eq:knownUndetermined}) respectively. Then it checks the satisfiability of $s\wedge edge\_constraint$. If the result is `sat', the model is retrieved. If $model$ remains equal to $\epsilon$ after this step, $\getModel(s)$ returns `Null'. Otherwise, it returns the first element of $model$ and removes it from $model$.

Returning to Line~\ref{line:mg-syn:nullEdge} in Algorithm~\ref{alg:mg-syn}, we can see the third difference: the handling of a `Null' return value from $\getEdge(s)$.  In this situation, the algorithm proceeds to check whether the current $s$ can be environment-winning directly. Specifically, if $Y_{potential\_swin}(s)=\emptyset$ holds, then $\noSwinPotential(s)$ returns $\trueVal$ and $s$ is determined to be as environment-winning.
\begin{equation}
\begin{aligned}
    Y_{potential\_swin}(s)=\{Y\in2^\Y\mid\,&\circled{1} Y\notin Y_{known\_ewin}(s)\text{ and }\\
    &\circled{2}X_{known\_swin}(s,Y)\cup X_{known\_undetermined}(s,Y)=2^\X\}
\end{aligned}
\end{equation}
Observing that $\getModel(s)$ returns `Null' implies that $s\wedge edge\_constraint$ is unsatisfiable. With $Y_{potential\_swin}(s)=\emptyset$, there are two scenarios for $Y\in2^\Y$. If $Y\in Y_{known\_ewin}(s)$ holds, as previously discussed, the environment can win the plays from $s$ with the system choice $Y$. In the case of $Y\notin Y_{known\_ewin}(s)$, it follows that there then exists $X\in 2^\X \setminus (X_{known\_swin}(s,Y) \cup X_{known\_undetermined}(s,Y))$ such that the formula $s\wedge X\wedge Y$ is unsatisfiable. This means that it can no longer reach an accepting transition starting from $s$ via the edge $X\cup Y$. Taking both scenarios into account, it is concluded that $s$ is an environment-winning state.

Lastly, we need to note the reset of $model$ to the empty trace. This occurs when encountering a previously visited but undetermined state (Line~\ref{line:mg-syn:resetRevisit}) or backtracking from a recursive call  (Line~\ref{line:mg-syn:resetBacktrack}). In situations of recursive expansion, specifically when there is no additional information and the algorithm can only proceed with the depth-first strategy (as illustrated in Figure~\ref{fig:approach-overview}-(a)), it moves forward following a single satisfiable trace.

\subsection{State Entailment}

The condition of whether a state is a known system/environment-winning state is checked multiple times in our approach. Here, we attempt to relax this condition. We can establish that it is sufficient to replace the condition mentioned above with a determination of whether a state is semantically entailed by a known system-winning state or semantically entails a known environment-winning state. An \LTLf satisfiability solver can help determine the semantic entailment relationship between two \LTLf formulas $\varphi_1$ and $\varphi_2$: $\varphi_1\Rightarrow \varphi_2$ holds iff $\varphi_1\wedge\neg \varphi_2$ is unsatisfiable.

\begin{lemma}\label{lem:realizableEntail}
Let $(\varphi_1,\X,\Y)_{Moore}$ and $(\varphi_2,\X,\Y)_{Moore}$ be two \ltlf specifications sharing the same input and output variables and $\varphi_1\Rightarrow\varphi_2$ holds.
\begin{enumerate}
\item If $(\varphi_1,\X,\Y)_{Moore}$ is realizable, then $(\varphi_2,\X,\Y)_{Moore}$ is realizable;\label{item:lem-swin-imple}

\item if $(\varphi_2,\X,\Y)_{Moore}$ is unrealizable, then $(\varphi_1,\X,\Y)_{Moore}$ is unrealizable.
\end{enumerate}
\end{lemma}
\begin{proof}
The two statements in this lemma are the contrapositives of each other, and therefore, they must both hold simultaneously. Hence, we only prove Statement~(\ref{item:lem-swin-imple}) here.

By Definition~\ref{def:syn}, if $(\varphi_1,\X,\Y)_{Moore}$ is realizable, then there exists a winning strategy $g: (2^\X)^* \to 2^\Y$ such that for an arbitrary infinite sequence $\lambda = X_0, X_1, \cdots \in (2^\X)^\omega$, there is $k \geq 0$ such that $\rho\models\varphi_1$ holds, where $\rho=(X_0\cup g(\epsilon)),(X_1\cup g(X_0)),\cdots,(X_k\cup g(X_0,\cdots,X_{k-1}))$. Notice we have $\varphi_1\Rightarrow\varphi_2$, so $\rho\models\varphi_2$ also holds. Thus $(\varphi_2,\X,\Y)_{Moore}$ is realizable.
\end{proof}
\begin{lemma}\label{lem:winningEntail}
Let $\G=(2^{\X\cup\Y},S,\delta,T)_{Moore}$ be a \tdfa games, $s_1\in S$ and $s_2\in S$ be two states such that $s_1\Rightarrow s_2$.
\begin{enumerate}
\item If $s_1$ is a system-winning state in $\G$, then $s_2$ is a system-winning state in $\G$;

\item if $s_2$ is an environment-winning state in $\G$, then $s_1$ is an environment-winning state in $\G$.
\end{enumerate}
\end{lemma}
\begin{proof}
Lemma~\ref{lem:syn2game} establishes the correspondence between system/environment-winning states and realizability. So we can get this lemma directly from Lemma \ref{lem:realizableEntail}.
\end{proof}

Table~\ref{tab:entail} provides a list of scenarios in our method where conditional expressions can be relaxed according to Lemma~\ref{lem:winningEntail}, which all appear in Algorithm~\ref{alg:checkCurrentStatus}. By leveraging these, we make fuller use of the identified system/environment-winning states and thus reduce our search space.

\begin{table}[H]
\caption{The conditional expressions in Algorithm~\ref{alg:checkCurrentStatus} that can be relaxed through entailment.}
\label{tab:entail}
\begin{tabular}{cll}
\hline
Line & Original Condition & Relaxed Condition \\ \hline
\ref{line:checkCurrent:sInSwin}    & $s\in swin\_state$                  & $\exists t\in swin\_state.t\Rightarrow s$             \\
\ref{line:checkCurrent:fpNotinSwin}    & $\fp{s,X\cup Y}\notin swin\_state$                  & $\forall t\in swin\_state.t\not\Rightarrow \fp{s,X\cup Y}$             \\
\ref{line:checkCurrent:sInEwin}    & $s\in ewin\_state$                  & $\exists t\in ewin\_state.s\Rightarrow t$             \\
\ref{line:checkCurrent:fpInEwin}    & $\fp{s,X\cup Y}\in\{s\}\cup ewin\_state$                  & $\fp{s,X\cup Y}=s$ or $\exists t\in ewin\_state.\fp{s,X\cup Y}\Rightarrow t$             \\ \hline
\end{tabular}
\end{table}

\section{Experimental Evaluation}

We provide experimental evidence that the on-the-fly approach offers the potential to avoid constructing complete automata and generally outperforms the backward search method.
We implement the on-the-fly \ltlf synthesis approach, as detailed in Section \ref{sec:on-the-fly} to \ref{sec:opt}, in a tool called \tool using C++ 11. The complete experimental setup, including the source code of \tool, the benchmarks, the compared tools, and the original logs produced during the experiment, is available at \cite{artifact}.

\subsection{Setup}

\paragraph{Compared Tools}
We evaluate the performance of our approach and \tool by comparing with the top three \ltlf synthesis tools from the latest reactive synthesis competition SYNTCOMP 2023~\cite{syntcomp,syntcomp23}: \lisa~\cite{BLTV20}, \lydia~\cite{DF21}, and \nike~\cite{mf23nike}. Both \lisa and \lydia are state-of-the-art \ltlf synthesis tools that are based on the backward search approach. \nike, which also performs forward synthesis, implements different heuristics from \tool for selecting the forward direction (as discussed in Section~\ref{sec:related-work}). Besides, \nike follows and enhances \cynthia \cite{GFLVXZ22}, another \ltlf synthesis tool that is excluded from our comparison due to its inferior performance in our preliminary experiments.
All three tools are run with their default parameters.

\paragraph{Benchmarks}
We collect, in total, 3380 \LTLf synthesis instances from literature: 1400 \emph{Random} instances \cite{ZTLPV17,BLTV20}, 140 \emph{Two-player-Games} instances \cite{TV19,BLTV20}, 40 \emph{Patterns} instances \cite{XLZSPV21}, and 1800 \emph{Ascending} instances \cite{XLHXLPSV24}.

\paragraph{Running Platform and Resources} We run the experiments on a CentOS 7.4 cluster, where each instance has exclusive access to a processor core of the Intel Xeon 6230 CPU running at 2.1 GHz, with 8 GB of memory and a 30-minute time limit. The execution time is measured with the Unix command \texttt{time}.

\subsection{Results and Discussion}

\subsubsection{Comparison with Baseline}

Table~\ref{tab:cmp-overall} shows the numbers of instances solved by different tools and approaches. The data in the `backward' column, merged from \lydia and \lisa, correspond to the backward search approach. The data in the `on-the-fly' column, merged from \nike and \tool, correspond to the on-the-fly approach. Figure~\ref{fig:solved-num} illustrates the number of instances that can be solved within different time limits. Figure~\ref{fig:state-cnt} compares the number of {\sf(T)DFA} states computed during the solving processes of \tool and \lisa in different instances.

\begin{table}[h]
\centering
\caption{Comparison among different tools and approaches. The data in the `backward' column are merged from \lydia and \lisa, and the data in the `on-the-fly' column are merged from \nike and \tool.}
\label{tab:cmp-overall}
% \small
\begin{tabular}{ll|rrrr|rr}
\hline
\multicolumn{2}{c|}{\multirow{2}{*}{}}                                & \multicolumn{4}{c|}{Compared by tools}                    & \multicolumn{2}{c}{Compared by approaches} \\ \cline{3-8} 
\multicolumn{2}{c|}{}                                                 & \lydia        & \lisa         & \nike        & \tool         & backward            & on-the-fly           \\ \hline
\multicolumn{1}{l|}{\multirow{2}{*}{\emph{Random}}}           & Realizable   & \textbf{356} & 351          & 351         & 336           & \textbf{361}                 & 354         \\
\multicolumn{1}{l|}{}                                  & Unrealizable & 920          & \textbf{965} & 842         & 791           & \textbf{965}        & 901                  \\ \cline{1-2}
\multicolumn{2}{l|}{\emph{Patterns}}                                         & \textbf{40}  & 38           & \textbf{40} & \textbf{40}   & \textbf{40}         & \textbf{40}          \\ \cline{1-2}
\multicolumn{1}{p{1.8cm}|}{\multirow{3}{=}{\emph{Two-player-Games}}} & \emph{s-counter}    & \textbf{12}  & 8            & 5           & 4             & \textbf{12}         & 5                    \\
\multicolumn{1}{l|}{}                                  & \emph{d-counters}   & \textbf{6}   & \textbf{6}   & 5           & \textbf{6}    & \textbf{6}          & \textbf{6}           \\
\multicolumn{1}{l|}{}                                  & \emph{nim}          & \textbf{20}  & 15           & 18          & 5             & \textbf{20}         & 18                   \\ \cline{1-2}
\multicolumn{1}{l|}{\multirow{2}{*}{\emph{Ascending}}}         & Realizable   & 1250         & 569          & 1302        & \textbf{1306} & 1251                & \textbf{1306}        \\
\multicolumn{1}{l|}{}                                  & Unrealizable & 234          & 210          & 216         & \textbf{365}  & 237                 & \textbf{368}         \\ \hline
\multicolumn{2}{l|}{Uniquely solved}                                  & 7            & 33           & 2           & \textbf{93}   & 112                  & \textbf{218}         \\ \cline{1-2}
\multicolumn{2}{l|}{Total}                                            & 2838         & 2162         & 2779        & \textbf{2853} & 2892                & \textbf{2998}        \\ \hline
\end{tabular}
\end{table}

\begin{figure}[htbp]
    \centering
    \begin{minipage}[b]{0.42\textwidth}
        \centering
        \includegraphics[width=\textwidth]{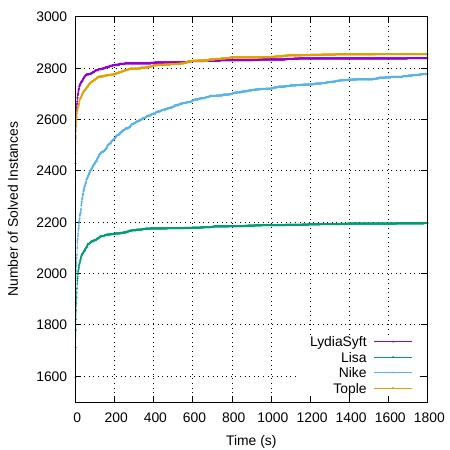}
        \caption{The number of instances solved by each tool over time.}
        \label{fig:solved-num}
        \Description{The number of solved instances by each tool over time.}
    \end{minipage}
    \hspace{0.4cm}
    \begin{minipage}[b]{0.42\textwidth}
        \centering
        \includegraphics[width=\textwidth]{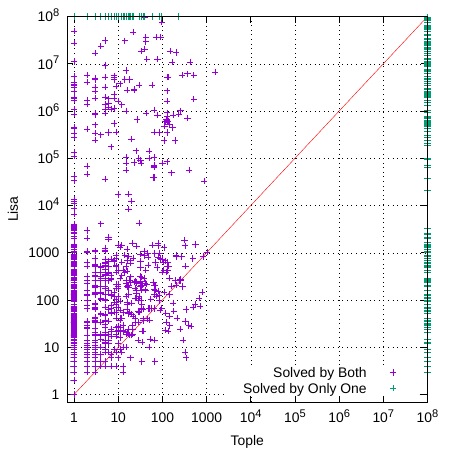}
        \caption{The number of computed {\sf(T)DFA} states in instances solved by \tool or \lisa.}
        \label{fig:state-cnt}
        \Description{The number of computed {\sf(T)DFA} states in instances solved by \tool or \lisa.}
    \end{minipage}
\end{figure}

The data and results suggest the following observations and discussions.
\begin{itemize}

\item \tool and on-the-fly approach show an overall advantage of the solving capability. They achieve the maximum number of uniquely solved instances (93/218), the maximum number of totally solved instances (2853/2998), and the highest endpoint in Figure~\ref{fig:solved-num}.

\item No tool or approach dominates in solving \LTLf synthesis problems. Firstly, we observe that each tool or approach can solve instances that others cannot, as evidenced by non-zero values in the 'Uniquely solved' row of Table~\ref{tab:cmp-overall}. This is expected given the overall complexity of the problem and the fact that different approaches adopt different heuristics for speeding up the search. Secondly, performance disparities among tools and approaches vary across different benchmark classes. Specifically, \lydia excels in solving realizable \emph{Random} instances and \emph{Two-player-Game} instances, \lisa solves the most unrealizable \emph{Random} instances, while \tool achieves optimal results in \emph{Ascending} instances. Meanwhile, the backward search approach outperforms in \emph{Random} instances and \emph{Two-player-Game} instances, whereas the on-the-fly approach performs better in \emph{Ascending} instances.

\item The number of computed states within the on-the-fly search is not always lower than that of the backward search approach. Both the backward search tools \lydia and \lisa construct a minimized \dfa. But \nike and \tool depend on the propositional equivalence to define the state space, which may not be minimized. To illustrate this, we compare \lisa and \tool in Figure~\ref{fig:state-cnt} (as shown in Table~\ref{tab:cmp-overall}, their performance across different benchmark classes tends to be complementary). As depicted in Figure~\ref{fig:state-cnt}, there are instances distributed below the red reference line, indicating that in some cases, a complete minimized deterministic automaton may have fewer states than a subset of states within a non-minimized automaton.

\item The searching policy has a considerable impact on the solving process of the on-the-fly approach. According to Table~\ref{tab:cmp-overall}, in unrealizable \emph{Random} instances, \nike and \tool solve 842 and 791 instances respectively, while they solve 901 unrealizable \emph{Random} instances in total. More specifically, among unrealizable \emph{Random} instances, \nike and \tool solve 732 the same instances, and each individually solves 110 and 59 respectively. This indicates that there are certain differences between the sets of successfully solved instances by \nike and \tool. We attribute this discrepancy to the fact that \nike and \tool are both on-the-fly tools but adopt different policies for selecting the moving direction during the depth-first search.

\end{itemize}

\subsubsection{Evaluation of Optimization Techniques within \tool}

\begin{table}[htbp]
\centering
\caption{Comparison of the application of two optimization techniques: model-guided synthesis and state entailment, represented by the options `-m' and `-e' respectively. Here, `0' indicates non-application and `1' indicates application of the respective technique. Abbreviations used are Realizable (Real.), Unrealizable (Unre.), \emph{s-counter} (\emph{s-count.}), and \emph{d-counters} (\emph{d-count.}).}
\label{tab:cmp-withinTople}
\begin{tabular}{l|rrrrrrrr|rr}
\hline
          & \multicolumn{2}{c|}{\textit{Random}}                                & \multicolumn{1}{c|}{\multirow{2}{*}{\textit{Patterns}}} & \multicolumn{3}{c|}{\textit{Two-player-Games}}                                                                         & \multicolumn{2}{c|}{\textit{Ascending}} & \multirow{2}{*}{Total} & \multirow{2}{*}{$\Delta$ Solve} \\ \cline{2-3} \cline{5-9}
          & \multicolumn{1}{r}{Real.} & \multicolumn{1}{r|}{Unre.} & \multicolumn{1}{c|}{}                                   & \multicolumn{1}{r}{\textit{s-count.}} & \multicolumn{1}{r}{\textit{d-count.}} & \multicolumn{1}{r|}{\textit{nim}} & Real.        & Unre.       &                        &                        \\ \hline
-m 0 -e 0 & 328                             & 736                               & 40                                                      & 4                                       & 6                                        & 5                                 & 1309              & 311                & 2739                   & 0                      \\ \hdashline
-m 1 -e 0 & 337                             & 759                               & 40                                                      & 4                                       & 6                                        & 4                                 & 1305              & 350                & 2805                   & +66                     \\
-m 0 -e 1 & 327                             & 760                               & 40                                                      & 4                                       & 6                                        & 5                                 & 1310              & 327                & 2779                   & +40                     \\
-m 1 -e 1 & 336                             & 791                               & 40                                                      & 4                                       & 6                                        & 5                                 & 1306              & 365                & 2853                   & +114                    \\ \hline
\end{tabular}
\end{table}

Table~\ref{tab:cmp-withinTople} compares the application of the optimization techniques: model-guided synthesis and state entailment. The last column, `$\Delta$ Solve', shows the improvement in the total number of solved instances, using the baseline configuration `-m 0 -e 0'. Both techniques enhance the overall solving capability of the on-the-fly approach. On the other hand, the numbers of solved instances with the configuration of  `-m 1 -e 0' in the realizable \emph{Ascending} instances (1305 compared to 1309 in the baseline) and with the configuration of `-m 0 -e 1' realizable \emph{Random} instances (327 compared to 328 in the baseline), are slightly less than those in the baseline. This implies that while the optimizations generally improve performance, there are also costs to using them, and the results are not always positive.

\section{Related Work}\label{sec:related-work}

\paragraph{Other Attempts at On-the-fly \LTLf Synthesis} This work builds upon previous efforts of \cite{XLZSPV21,XLHXLPSV24}. Concurrently, other studies follow \cite{XLZSPV21} also employ the on-the-fly approach to solving \ltlf synthesis problem. De Giacomo et al. \cite{GFLVXZ22} formulates the searched state space of \dfa games as And/Or graphs~\cite{N71}, which is widely used in automated planning with nondeterministic models~\cite{GNT17nondeterm}. This study also introduces Sequential Decision Diagrams ({\SDD}s) as a data structure for reasoning through And/Or graphs. Despite the improved performance in certain benchmarks, this approach encounters significant scalability issues. Within the same framework in \cite{XLZSPV21,GFLVXZ22}, \cite{mf23nike} devise a procedure inspired by the Davis-Putnam-Logemann-Loveland algorithm~\cite{DP60,DLL62} to improve edge enumeration and introduces a syntactic equivalence check for search states.

\paragraph{\ltl Synthesis}
The first work to consider \LTL synthesis is \cite{PR89}, which solves the synthesis problem by reducing it to a Rabin game~\cite{EJ88}. This approach constructs a non-deterministic B\"uchi automaton from the input \LTL formula, and then determinizes it to its equivalent Rabin automaton, a process which takes worst-case double-exponential time. The complexity of solving a Rabin game is NP-Complete~\cite{EJ88}. Nowadays, the standard approach is to reduce \LTL synthesis to the parity game~\cite{EJ91}, because a parity game can be solved in quasi-polynomial time~\cite{CJKLS17}, even though the doubly-exponential process to obtain a deterministic parity automaton cannot be avoided. \LTL synthesis tools like \ltlsynt~\cite{MC18} and \strix~\cite{MSL18}, are built using the parity-game approach. Because of the challenge to determinize an $\omega$ automaton, researchers also consider other possibilities, e.g., by reducing \LTL synthesis to the bounded safety game~\cite{Kup06a}. \acacia~\cite{BFJ12} is a representative tool following the safety-game approach. The annual reactive synthesis competition~\cite{syntcomp} drives progress in this field, yet the scalability issue is still a major problem.

\section{Concluding Remarks}

We have presented an on-the-fly approach framework for synthesizing \ltlf specifications. By concurrently conducting synthesis and constructing automata, we get the chance to bypass the double exponential growth of state space. An empirical comparison of this method to state-of-the-art \ltlf synthesizers suggests that it can achieve the best overall performance.
Several future research directions are being considered.
Firstly, to further reduce state space to be searched, it would be interesting to investigate how to convert \ltlf to minimized {\sf (T)DFA} on the fly. Secondly, the design of low-cost and effective heuristics to guide the search process could significantly enhance the efficiency of the on-the-fly approach. Lastly, beyond the framework introduced in this article, the on-the-fly approach could be integrated with compositional techniques as in \cite{BLTV20,BDDLVZ22}.

%%
%% The acknowledgments section is defined using the "acks" environment
%% (and NOT an unnumbered section). This ensures the proper
%% identification of the section in the article metadata, and the
%% consistent spelling of the heading.
% \begin{acks}
% To Robert, for the bagels and explaining CMYK and color spaces.
% \end{acks}

%%
%% The next two lines define the bibliography style to be used, and
%% the bibliography file.
\bibliographystyle{ACM-Reference-Format}
\bibliography{ok,cav}

%%% -*-BibTeX-*-
%%% Do NOT edit. File created by BibTeX with style
%%% ACM-Reference-Format-Journals [18-Jan-2012].

\begin{thebibliography}{46}

%%% ====================================================================
%%% NOTE TO THE USER: you can override these defaults by providing
%%% customized versions of any of these macros before the \bibliography
%%% command.  Each of them MUST provide its own final punctuation,
%%% except for \shownote{}, \showDOI{}, and \showURL{}.  The latter two
%%% do not use final punctuation, in order to avoid confusing it with
%%% the Web address.
%%%
%%% To suppress output of a particular field, define its macro to expand
%%% to an empty string, or better, \unskip, like this:
%%%
%%% \newcommand{\showDOI}[1]{\unskip}   % LaTeX syntax
%%%
%%% \def \showDOI #1{\unskip}           % plain TeX syntax
%%%
%%% ====================================================================

\ifx \showCODEN    \undefined \def \showCODEN     #1{\unskip}     \fi
\ifx \showDOI      \undefined \def \showDOI       #1{#1}\fi
\ifx \showISBNx    \undefined \def \showISBNx     #1{\unskip}     \fi
\ifx \showISBNxiii \undefined \def \showISBNxiii  #1{\unskip}     \fi
\ifx \showISSN     \undefined \def \showISSN      #1{\unskip}     \fi
\ifx \showLCCN     \undefined \def \showLCCN      #1{\unskip}     \fi
\ifx \shownote     \undefined \def \shownote      #1{#1}          \fi
\ifx \showarticletitle \undefined \def \showarticletitle #1{#1}   \fi
\ifx \showURL      \undefined \def \showURL       {\relax}        \fi
% The following commands are used for tagged output and should be
% invisible to TeX
\providecommand\bibfield[2]{#2}
\providecommand\bibinfo[2]{#2}
\providecommand\natexlab[1]{#1}
\providecommand\showeprint[2][]{arXiv:#2}

\bibitem[art(2024)]%
        {artifact}
 \bibinfo{year}{2024}\natexlab{}.
\newblock \bibinfo{booktitle}{\emph{Artifact for this article}}.
\newblock
\urldef\tempurl%
\url{https://drive.google.com/file/d/1JwH-Szs-dJ5KeZqV8159SL4gQSxL03VM/view?usp=sharing}
\showURL{%
\tempurl}


\bibitem[Bacchus and Kabanza(1998)]%
        {BK98}
\bibfield{author}{\bibinfo{person}{Fahiem Bacchus} {and} \bibinfo{person}{Froduald Kabanza}.} \bibinfo{year}{1998}\natexlab{}.
\newblock \showarticletitle{Planning for Temporally Extended Goals}.
\newblock \bibinfo{journal}{\emph{Annals of Mathematics and Artificial Intelligence}}  \bibinfo{volume}{22} (\bibinfo{year}{1998}), \bibinfo{pages}{5--27}.
\newblock


\bibitem[Bansal et~al\mbox{.}(2022)]%
        {BDDLVZ22}
\bibfield{author}{\bibinfo{person}{Suguman Bansal}, \bibinfo{person}{Giuseppe~De Giacomo}, \bibinfo{person}{Antonio~Di Stasio}, \bibinfo{person}{Yong Li}, \bibinfo{person}{Moshe~Y. Vardi}, {and} \bibinfo{person}{Shufang Zhu}.} \bibinfo{year}{2022}\natexlab{}.
\newblock \showarticletitle{Compositional Safety LTL Synthesis}. In \bibinfo{booktitle}{\emph{Verified Software: Theories, Tools, and Experiments (VSTTE)}}.
\newblock


\bibitem[Bansal et~al\mbox{.}(2020)]%
        {BLTV20}
\bibfield{author}{\bibinfo{person}{Suguman Bansal}, \bibinfo{person}{Yong Li}, \bibinfo{person}{Lucas Tabajara}, {and} \bibinfo{person}{Moshe Vardi}.} \bibinfo{year}{2020}\natexlab{}.
\newblock \showarticletitle{Hybrid compositional reasoning for reactive synthesis from finite-horizon specifications}. In \bibinfo{booktitle}{\emph{The Thirty-Fourth {AAAI} Conference on Artificial Intelligence}}. \bibinfo{publisher}{{AAAI} Press}, \bibinfo{pages}{9766--9774}.
\newblock


\bibitem[Bienvenu et~al\mbox{.}(2011)]%
        {BFM11}
\bibfield{author}{\bibinfo{person}{Meghyn Bienvenu}, \bibinfo{person}{Christian Fritz}, {and} \bibinfo{person}{Sheila~A McIlraith}.} \bibinfo{year}{2011}\natexlab{}.
\newblock \showarticletitle{Specifying and computing preferred plans}.
\newblock \bibinfo{journal}{\emph{Artificial Intelligence}}  \bibinfo{volume}{175} (\bibinfo{year}{2011}), \bibinfo{pages}{1308--1345}.
\newblock


\bibitem[Bloem et~al\mbox{.}(2018)]%
        {BCJ18}
\bibfield{author}{\bibinfo{person}{Roderick Bloem}, \bibinfo{person}{Krishnendu Chatterjee}, {and} \bibinfo{person}{Barbara Jobstmann}.} \bibinfo{year}{2018}\natexlab{}.
\newblock \bibinfo{booktitle}{\emph{Graph games and reactive synthesis}}.
\newblock \bibinfo{publisher}{Springer}, \bibinfo{address}{Cham}, \bibinfo{pages}{921--962}.
\newblock


\bibitem[Bloem et~al\mbox{.}(2012)]%
        {BJPPS12}
\bibfield{author}{\bibinfo{person}{Roderick Bloem}, \bibinfo{person}{Barbara Jobstmann}, \bibinfo{person}{Nir Piterman}, \bibinfo{person}{Amir Pnueli}, {and} \bibinfo{person}{Yaniv Saar}.} \bibinfo{year}{2012}\natexlab{}.
\newblock \showarticletitle{Synthesis of Reactive(1) designs}.
\newblock \bibinfo{journal}{\emph{J. Comput. System Sci.}} \bibinfo{volume}{78}, \bibinfo{number}{3} (\bibinfo{year}{2012}), \bibinfo{pages}{911--938}.
\newblock
\newblock
\shownote{In Commemoration of Amir Pnueli}.


\bibitem[Bohy et~al\mbox{.}(2012)]%
        {BFJ12}
\bibfield{author}{\bibinfo{person}{Aaron Bohy}, \bibinfo{person}{Emmanuel Filiot}, {and} \bibinfo{person}{Naiyong Jin}.} \bibinfo{year}{2012}\natexlab{}.
\newblock \showarticletitle{Acacia+, a tool for ltl synthesis}. In \bibinfo{booktitle}{\emph{of Lecture Notes in Computer Science}}. \bibinfo{publisher}{Springer-Verlag}, \bibinfo{pages}{652--657}.
\newblock


\bibitem[B{\"u}chi and Landweber(1969)]%
        {BL69}
\bibfield{author}{\bibinfo{person}{J~Richard B{\"u}chi} {and} \bibinfo{person}{Lawrence~H Landweber}.} \bibinfo{year}{1969}\natexlab{}.
\newblock \showarticletitle{Solving sequential conditions by finite-state strategies}.
\newblock \bibinfo{journal}{\emph{Trans. AMS}}  \bibinfo{volume}{138} (\bibinfo{year}{1969}), \bibinfo{pages}{295--311}.
\newblock


\bibitem[Calude et~al\mbox{.}(2017)]%
        {CJKLS17}
\bibfield{author}{\bibinfo{person}{Cristian~S. Calude}, \bibinfo{person}{Sanjay Jain}, \bibinfo{person}{Bakhadyr Khoussainov}, \bibinfo{person}{Wei Li}, {and} \bibinfo{person}{Frank Stephan}.} \bibinfo{year}{2017}\natexlab{}.
\newblock \showarticletitle{Deciding Parity Games in Quasipolynomial Time}. In \bibinfo{booktitle}{\emph{Proceedings of the 49th Annual ACM SIGACT Symposium on Theory of Computing}} (Montreal, Canada) \emph{(\bibinfo{series}{STOC 2017})}. \bibinfo{publisher}{Association for Computing Machinery}, \bibinfo{address}{New York, NY, USA}, \bibinfo{pages}{252–263}.
\newblock
\showISBNx{9781450345286}


\bibitem[Church(1962)]%
        {Chu62}
\bibfield{author}{\bibinfo{person}{Alonzo Church}.} \bibinfo{year}{1962}\natexlab{}.
\newblock \showarticletitle{Logic, arithmetics, and automata}. In \bibinfo{booktitle}{\emph{Proceedings of the international congress of mathematicians}}, Vol.~\bibinfo{volume}{1962}. \bibinfo{publisher}{Institut Mittag-Leffler}, \bibinfo{pages}{23--35}.
\newblock


\bibitem[Davis et~al\mbox{.}(1962)]%
        {DLL62}
\bibfield{author}{\bibinfo{person}{M. Davis}, \bibinfo{person}{G. Logemann}, {and} \bibinfo{person}{D. Loveland}.} \bibinfo{year}{1962}\natexlab{}.
\newblock \showarticletitle{A machine program for theorem-proving}.
\newblock \bibinfo{journal}{\emph{Commun. ACM}} \bibinfo{volume}{5}, \bibinfo{number}{7} (\bibinfo{date}{July} \bibinfo{year}{1962}), \bibinfo{pages}{394--397}.
\newblock


\bibitem[Davis and Putnam(1960)]%
        {DP60}
\bibfield{author}{\bibinfo{person}{M. Davis} {and} \bibinfo{person}{H. Putnam}.} \bibinfo{year}{1960}\natexlab{}.
\newblock \showarticletitle{A Computing Procedure for Quantification Theory}.
\newblock \bibinfo{journal}{\emph{J. ACM}} \bibinfo{volume}{7}, \bibinfo{number}{3} (\bibinfo{date}{July} \bibinfo{year}{1960}), \bibinfo{pages}{201--215}.
\newblock


\bibitem[De~Giacomo and Favorito(2021)]%
        {DF21}
\bibfield{author}{\bibinfo{person}{Giuseppe De~Giacomo} {and} \bibinfo{person}{Marco Favorito}.} \bibinfo{year}{2021}\natexlab{}.
\newblock \showarticletitle{Compositional approach to translate LTLf/LDLf into deterministic finite automata}. In \bibinfo{booktitle}{\emph{Proceedings of the International Conference on Automated Planning and Scheduling}}, Vol.~\bibinfo{volume}{31}. \bibinfo{publisher}{AAAI Press}, \bibinfo{pages}{122--130}.
\newblock


\bibitem[De~Giacomo and Vardi(2013)]%
        {GV13}
\bibfield{author}{\bibinfo{person}{Giuseppe De~Giacomo} {and} \bibinfo{person}{Moshe~Y Vardi}.} \bibinfo{year}{2013}\natexlab{}.
\newblock \showarticletitle{Linear Temporal Logic and Linear Dynamic Logic on Finite Traces}. In \bibinfo{booktitle}{\emph{Proceedings of the Twenty-Third international joint conference on Artificial Intelligence}}. \bibinfo{publisher}{AAAI Press}, \bibinfo{pages}{854--860}.
\newblock


\bibitem[De~Giacomo and Vardi(2015)]%
        {GV15}
\bibfield{author}{\bibinfo{person}{Giuseppe De~Giacomo} {and} \bibinfo{person}{Moshe~Y Vardi}.} \bibinfo{year}{2015}\natexlab{}.
\newblock \showarticletitle{Synthesis for LTL and LDL on Finite Traces}. In \bibinfo{booktitle}{\emph{Proceedings of the 24th International Conference on Artificial Intelligence}}. \bibinfo{publisher}{AAAI Press}, \bibinfo{pages}{1558--1564}.
\newblock


\bibitem[Emerson(1990)]%
        {Eme90}
\bibfield{author}{\bibinfo{person}{E.A. Emerson}.} \bibinfo{year}{1990}\natexlab{}.
\newblock \showarticletitle{Temporal and modal logic}.
\newblock In \bibinfo{booktitle}{\emph{Handbook of {T}heoretical {C}omputer {S}cience}}, \bibfield{editor}{\bibinfo{person}{J.~Van Leeuwen}} (Ed.). Vol.~\bibinfo{volume}{B}. \bibinfo{publisher}{Elsevier, {MIT} Press}, Chapter~16, \bibinfo{pages}{997--1072}.
\newblock


\bibitem[Emerson and Jutla(1988)]%
        {EJ88}
\bibfield{author}{\bibinfo{person}{E.A. Emerson} {and} \bibinfo{person}{C. Jutla}.} \bibinfo{year}{1988}\natexlab{}.
\newblock \showarticletitle{The Complexity of Tree Automata and Logics of Programs}. In \bibinfo{booktitle}{\emph{Proc.\ 29th IEEE Symp. on Foundations of Computer Science}}. \bibinfo{pages}{328--337}.
\newblock


\bibitem[Emerson and Jutla(1991)]%
        {EJ91}
\bibfield{author}{\bibinfo{person}{E.A. Emerson} {and} \bibinfo{person}{C. Jutla}.} \bibinfo{year}{1991}\natexlab{}.
\newblock \showarticletitle{Tree Automata, $\mu$-Calculus and Determinacy}. In \bibinfo{booktitle}{\emph{Proc.\ 32nd IEEE Symp. on Foundations of Computer Science}}. \bibinfo{pages}{368--377}.
\newblock


\bibitem[Esparza et~al\mbox{.}(2020)]%
        {EKS20}
\bibfield{author}{\bibinfo{person}{Javier Esparza}, \bibinfo{person}{Jan K\v{r}et\'{\i}nsk\'{y}}, {and} \bibinfo{person}{Salomon Sickert}.} \bibinfo{year}{2020}\natexlab{}.
\newblock \showarticletitle{A Unified Translation of Linear Temporal Logic to $\omega$-Automata}.
\newblock \bibinfo{journal}{\emph{J. ACM}} \bibinfo{volume}{67}, \bibinfo{number}{6}, Article \bibinfo{articleno}{33} (\bibinfo{year}{2020}), \bibinfo{numpages}{61}~pages.
\newblock


\bibitem[Favorito(2023)]%
        {mf23nike}
\bibfield{author}{\bibinfo{person}{Marco Favorito}.} \bibinfo{year}{2023}\natexlab{}.
\newblock \bibinfo{title}{Forward LTLf Synthesis: DPLL At Work}.
\newblock
\newblock
\showeprint[arxiv]{2302.13825}


\bibitem[Filiot et~al\mbox{.}(2009)]%
        {FJR09}
\bibfield{author}{\bibinfo{person}{Emmanuel Filiot}, \bibinfo{person}{Naiyong Jin}, {and} \bibinfo{person}{Jean-Fran{\c{c}}ois Raskin}.} \bibinfo{year}{2009}\natexlab{}.
\newblock \showarticletitle{An antichain algorithm for LTL realizability}. In \bibinfo{booktitle}{\emph{Computer Aided Verification: 21st International Conference, CAV 2009, Grenoble, France, June 26-July 2, 2009. Proceedings 21}}. \bibinfo{publisher}{Springer}, \bibinfo{pages}{263--277}.
\newblock


\bibitem[Ghallab et~al\mbox{.}(2016)]%
        {GNT17nondeterm}
\bibfield{author}{\bibinfo{person}{Malik Ghallab}, \bibinfo{person}{Dana Nau}, {and} \bibinfo{person}{Paolo Traverso}.} \bibinfo{year}{2016}\natexlab{}.
\newblock \showarticletitle{Deliberation with Nondeterministic Models}.
\newblock In \bibinfo{booktitle}{\emph{Automated planning and acting}}. \bibinfo{publisher}{Cambridge University Press}, \bibinfo{pages}{197--276}.
\newblock


\bibitem[Giacomo et~al\mbox{.}(2022)]%
        {GFLVXZ22}
\bibfield{author}{\bibinfo{person}{Giuseppe~De Giacomo}, \bibinfo{person}{Marco Favorito}, \bibinfo{person}{Jianwen Li}, \bibinfo{person}{Moshe~Y. Vardi}, \bibinfo{person}{Shengping Xiao}, {and} \bibinfo{person}{Shufang Zhu}.} \bibinfo{year}{2022}\natexlab{}.
\newblock \showarticletitle{LTLf Synthesis as AND-OR Graph Search: Knowledge Compilation at Work}. In \bibinfo{booktitle}{\emph{Proceedings of the Thirty-First International Joint Conference on Artificial Intelligence}}. \bibinfo{publisher}{AAAI Press}, \bibinfo{pages}{3292--3298}.
\newblock


\bibitem[Jacobs et~al\mbox{.}(2023a)]%
        {syntcomp23}
\bibfield{author}{\bibinfo{person}{Swen Jacobs}, \bibinfo{person}{Guillermo Perez}, {and} \bibinfo{person}{Philipp Schlehuber-Caissier}.} \bibinfo{year}{2023}\natexlab{a}.
\newblock \bibinfo{booktitle}{\emph{Data, scripts, and results from SYNTCOMP 2023}}.
\newblock
\urldef\tempurl%
\url{https://doi.org/10.5281/zenodo.8112518}
\showDOI{\tempurl}


\bibitem[Jacobs et~al\mbox{.}(2023b)]%
        {syntcomp}
\bibfield{author}{\bibinfo{person}{Swen Jacobs}, \bibinfo{person}{Guillermo~A. Pérez}, {and} \bibinfo{person}{Belgium~Philipp Schlehuber-Caissier}.} \bibinfo{year}{2023}\natexlab{b}.
\newblock \bibinfo{title}{The Reactive Synthesis Competition}.
\newblock \bibinfo{howpublished}{\url{http://www.syntcomp.org/}}.
\newblock


\bibitem[Klarlund and M{\o}ller(2001)]%
        {MonaManual2001}
\bibfield{author}{\bibinfo{person}{Nils Klarlund} {and} \bibinfo{person}{Anders M{\o}ller}.} \bibinfo{year}{2001}\natexlab{}.
\newblock \bibinfo{booktitle}{\emph{{MONA Version 1.4 User Manual}}}.
\newblock BRICS, Department of Computer Science, University of Aarhus.
\newblock
\newblock
\shownote{Notes Series NS-01-1. Available from \texttt{\small http://www.brics.dk/mona/}}.


\bibitem[Kupferman(2006)]%
        {Kup06a}
\bibfield{author}{\bibinfo{person}{O. Kupferman}.} \bibinfo{year}{2006}\natexlab{}.
\newblock \showarticletitle{Avoiding Determinization}. In \bibinfo{booktitle}{\emph{Proc.\ 21st IEEE Symp. on Logic in Computer Science}}. \bibinfo{pages}{243--254}.
\newblock


\bibitem[Lahijanian et~al\mbox{.}(2015)]%
        {LAFKV15}
\bibfield{author}{\bibinfo{person}{Morteza Lahijanian}, \bibinfo{person}{Shaull Almagor}, \bibinfo{person}{Dror Fried}, \bibinfo{person}{Lydia Kavraki}, {and} \bibinfo{person}{Moshe Vardi}.} \bibinfo{year}{2015}\natexlab{}.
\newblock \showarticletitle{This time the robot settles for a cost: A quantitative approach to temporal logic planning with partial satisfaction}. In \bibinfo{booktitle}{\emph{Proceedings of the AAAI Conference on Artificial Intelligence}}. \bibinfo{publisher}{AAAI Press}, \bibinfo{pages}{3664--3671}.
\newblock


\bibitem[Li et~al\mbox{.}(2019)]%
        {LRPZV19}
\bibfield{author}{\bibinfo{person}{Jianwen Li}, \bibinfo{person}{Kristin~Y. Rozier}, \bibinfo{person}{Geguang Pu}, \bibinfo{person}{Yueling Zhang}, {and} \bibinfo{person}{Moshe~Y. Vardi}.} \bibinfo{year}{2019}\natexlab{}.
\newblock \showarticletitle{SAT-Based Explicit LTLf Satisfiability Checking}. In \bibinfo{booktitle}{\emph{The Thirty-Third {AAAI} Conference on Artificial Intelligence}}, Vol.~\bibinfo{volume}{33}. \bibinfo{publisher}{{AAAI} Press}, \bibinfo{pages}{2946--2953}.
\newblock


\bibitem[Meyer et~al\mbox{.}(2018)]%
        {MSL18}
\bibfield{author}{\bibinfo{person}{Philipp~J. Meyer}, \bibinfo{person}{Salomon Sickert}, {and} \bibinfo{person}{Michael Luttenberger}.} \bibinfo{year}{2018}\natexlab{}.
\newblock \showarticletitle{Strix: Explicit Reactive Synthesis Strikes Back!}. In \bibinfo{booktitle}{\emph{Computer Aided Verification - 30th International Conference, Proceedings, Part {I}}} \emph{(\bibinfo{series}{Lecture Notes in Computer Science}, Vol.~\bibinfo{volume}{10981})}. \bibinfo{publisher}{Springer}, \bibinfo{pages}{578--586}.
\newblock


\bibitem[Michaud and Colange(2018)]%
        {MC18}
\bibfield{author}{\bibinfo{person}{Thibaud Michaud} {and} \bibinfo{person}{Maximilien Colange}.} \bibinfo{year}{2018}\natexlab{}.
\newblock \showarticletitle{Reactive Synthesis from LTL Specification with Spot}. In \bibinfo{booktitle}{\emph{Proceedings Seventh Workshop on Synthesis, SYNT@CAV 2018}} \emph{(\bibinfo{series}{Electronic Proceedings in Theoretical Computer Science})}.
\newblock


\bibitem[Nllsson(1971)]%
        {N71}
\bibfield{author}{\bibinfo{person}{Nills~J Nllsson}.} \bibinfo{year}{1971}\natexlab{}.
\newblock \bibinfo{title}{Problem Solving Methods in A rtificial I ntelligence}.
\newblock
\newblock


\bibitem[Pe{\v{s}}i{\'c} et~al\mbox{.}(2010)]%
        {PBV10}
\bibfield{author}{\bibinfo{person}{Maja Pe{\v{s}}i{\'c}}, \bibinfo{person}{Dragan Bo{\v{s}}na{\v{c}}ki}, {and} \bibinfo{person}{Wil~MP van~der Aalst}.} \bibinfo{year}{2010}\natexlab{}.
\newblock \showarticletitle{Enacting declarative languages using LTL: avoiding errors and improving performance}. In \bibinfo{booktitle}{\emph{Model Checking Software: 17th International SPIN Workshop}}. \bibinfo{publisher}{Springer}, \bibinfo{pages}{146--161}.
\newblock


\bibitem[Pnueli(1977)]%
        {Pnu77}
\bibfield{author}{\bibinfo{person}{Amir Pnueli}.} \bibinfo{year}{1977}\natexlab{}.
\newblock \showarticletitle{The temporal logic of programs}. In \bibinfo{booktitle}{\emph{18th Annual Symposium on Foundations of Computer Science}}. \bibinfo{publisher}{IEEE}, \bibinfo{pages}{46--57}.
\newblock
\urldef\tempurl%
\url{https://doi.org/10.1109/SFCS.1977.32}
\showDOI{\tempurl}


\bibitem[Pnueli and Rosner(1989)]%
        {PR89}
\bibfield{author}{\bibinfo{person}{Amir Pnueli} {and} \bibinfo{person}{Roni Rosner}.} \bibinfo{year}{1989}\natexlab{}.
\newblock \showarticletitle{On the Synthesis of an Asynchronous Reactive Module}. In \bibinfo{booktitle}{\emph{Proceedings of the 16th International Colloquium on Automata, Languages and Programming}} \emph{(\bibinfo{series}{ICALP '89})}. \bibinfo{publisher}{Springer-Verlag}, \bibinfo{address}{Berlin, Heidelberg}, \bibinfo{pages}{652–671}.
\newblock
\showISBNx{354051371X}


\bibitem[Rabin(1969)]%
        {Rab69}
\bibfield{author}{\bibinfo{person}{Michael~O Rabin}.} \bibinfo{year}{1969}\natexlab{}.
\newblock \showarticletitle{Decidability of Second Order Theories and Automata on Infinite Trees}.
\newblock \bibinfo{journal}{\emph{Transaction of the AMS}}  \bibinfo{volume}{141} (\bibinfo{year}{1969}), \bibinfo{pages}{1--35}.
\newblock


\bibitem[Shi et~al\mbox{.}(2020)]%
        {SXLGP20}
\bibfield{author}{\bibinfo{person}{Yingying Shi}, \bibinfo{person}{Shengping Xiao}, \bibinfo{person}{Jianwen Li}, \bibinfo{person}{Jian Guo}, {and} \bibinfo{person}{Geguang Pu}.} \bibinfo{year}{2020}\natexlab{}.
\newblock \showarticletitle{SAT-Based Automata Construction for LTL over Finite Traces}. In \bibinfo{booktitle}{\emph{27th Asia-Pacific Software Engineering Conference (APSEC)}}. \bibinfo{publisher}{IEEE}, \bibinfo{pages}{1--10}.
\newblock


\bibitem[Tabajara and Vardi(2019)]%
        {TV19}
\bibfield{author}{\bibinfo{person}{Lucas~M Tabajara} {and} \bibinfo{person}{Moshe~Y Vardi}.} \bibinfo{year}{2019}\natexlab{}.
\newblock \showarticletitle{Partitioning Techniques in LTLf Synthesis}. In \bibinfo{booktitle}{\emph{Proceedings of the 28th International Joint Conference on Artificial Intelligence}} \emph{(\bibinfo{series}{IJCAI 19})}. \bibinfo{publisher}{AAAI Press}, \bibinfo{pages}{5599--5606}.
\newblock


\bibitem[Tarjan(1972)]%
        {tarjan1972SCC}
\bibfield{author}{\bibinfo{person}{Robert Tarjan}.} \bibinfo{year}{1972}\natexlab{}.
\newblock \showarticletitle{Depth-first search and linear graph algorithms}.
\newblock \bibinfo{journal}{\emph{SIAM journal on computing}} \bibinfo{volume}{1}, \bibinfo{number}{2} (\bibinfo{year}{1972}), \bibinfo{pages}{146--160}.
\newblock


\bibitem[Tarski(1955)]%
        {tarski1955lattice}
\bibfield{author}{\bibinfo{person}{Alfred Tarski}.} \bibinfo{year}{1955}\natexlab{}.
\newblock \showarticletitle{A lattice-theoretical fixpoint theorem and its application}.
\newblock \bibinfo{journal}{\emph{Pacific J. Math.}}  \bibinfo{volume}{5} (\bibinfo{year}{1955}), \bibinfo{pages}{285--309}.
\newblock


\bibitem[Thomas(1997)]%
        {Tho97}
\bibfield{author}{\bibinfo{person}{Wolfgang Thomas}.} \bibinfo{year}{1997}\natexlab{}.
\newblock \showarticletitle{Languages, Automata, and Logic}.
\newblock \bibinfo{journal}{\emph{Handbook of {F}ormal {L}anguage {T}heory}}  \bibinfo{volume}{III} (\bibinfo{year}{1997}), \bibinfo{pages}{389--455}.
\newblock


\bibitem[van Der~Aalst et~al\mbox{.}(2009)]%
        {VPS09}
\bibfield{author}{\bibinfo{person}{Wil~MP van Der~Aalst}, \bibinfo{person}{Maja Pesic}, {and} \bibinfo{person}{Helen Schonenberg}.} \bibinfo{year}{2009}\natexlab{}.
\newblock \showarticletitle{Declarative workflows: Balancing between flexibility and support}.
\newblock \bibinfo{journal}{\emph{Computer Science-Research and Development}}  \bibinfo{volume}{23} (\bibinfo{year}{2009}), \bibinfo{pages}{99--113}.
\newblock


\bibitem[Xiao et~al\mbox{.}(2021)]%
        {XLZSPV21}
\bibfield{author}{\bibinfo{person}{Shengping Xiao}, \bibinfo{person}{Jianwen Li}, \bibinfo{person}{Shufang Zhu}, \bibinfo{person}{Yingying Shi}, \bibinfo{person}{Geguang Pu}, {and} \bibinfo{person}{Moshe~Y. Vardi}.} \bibinfo{year}{2021}\natexlab{}.
\newblock \showarticletitle{On the fly synthesis for LTL over finite traces}. In \bibinfo{booktitle}{\emph{The Thirty-Fourth {AAAI} Conference on Artificial Intelligence}}. \bibinfo{publisher}{{AAAI} Press}, \bibinfo{pages}{6530--6537}.
\newblock


\bibitem[Xiao et~al\mbox{.}(2024)]%
        {XLHXLPSV24}
\bibfield{author}{\bibinfo{person}{Shengping Xiao}, \bibinfo{person}{Yongkang Li}, \bibinfo{person}{Xinyue Huang}, \bibinfo{person}{Yicong Xu}, \bibinfo{person}{Jianwen Li}, \bibinfo{person}{Geguang Pu}, \bibinfo{person}{Ofer Strichman}, {and} \bibinfo{person}{Moshe~Y Vardi}.} \bibinfo{year}{2024}\natexlab{}.
\newblock \showarticletitle{Model-Guided Synthesis for LTL over Finite Traces}. In \bibinfo{booktitle}{\emph{International Conference on Verification, Model Checking, and Abstract Interpretation}}. \bibinfo{publisher}{Springer}, \bibinfo{pages}{186--207}.
\newblock


\bibitem[Zhu et~al\mbox{.}(2017)]%
        {ZTLPV17}
\bibfield{author}{\bibinfo{person}{Shufang Zhu}, \bibinfo{person}{Lucas~M Tabajara}, \bibinfo{person}{Jianwen Li}, \bibinfo{person}{Geguang Pu}, {and} \bibinfo{person}{Moshe~Y Vardi}.} \bibinfo{year}{2017}\natexlab{}.
\newblock \showarticletitle{Symbolic LTLf Synthesis}. In \bibinfo{booktitle}{\emph{Proceedings of the 26th International Joint Conference on Artificial Intelligence}}. \bibinfo{publisher}{AAAI Press}, \bibinfo{pages}{1362--1369}.
\newblock


\end{thebibliography}

%%
%% If your work has an appendix, this is the place to put it.
% \appendix

% \input{10-Appendix}

\end{document}